%% file: iclr2027_conference.tex
\documentclass{article}
\usepackage{iclr2027_conference,times}
\usepackage{titletoc} % Appendix contents only.
\usepackage{amsmath,amsfonts,bm}
\usepackage{hyperref,url,booktabs,tabularx,array,xcolor,placeins,graphicx,multirow,wrapfig}
\definecolor{placeholder}{RGB}{175,25,35}
\definecolor{draftblue}{RGB}{30,65,140}

\definecolor{guess}{RGB}{130,75,0}

\usepackage[normalem]{ulem}
\usepackage{capt-of}

\newtheorem{theorem}{Theorem}[section]

\newtheorem{lemma}[theorem]{Lemma}
\newtheorem{corollary}[theorem]{Corollary}

\title{Scale-Split Neural Operator for Memory- and Data-Efficient 3D Turbulence Prediction}

\author{%
\makebox[\dimexpr\textwidth-2\tabcolsep\relax][c]{%
Shaoxiang Qin$^{1,2}$, Yucheng Zhao$^{1}$, Zongyi Li$^{3}$, Liangzhu Leon Wang$^{4}$, Xiongye Xiao$^{1}$\thanks{Corresponding author: Xiongye Xiao \texttt{(xxiao9@utk.edu)}}}\\[2pt]
\makebox[\dimexpr\textwidth-2\tabcolsep\relax][c]{%
$^{1}$University of Tennessee, Knoxville \qquad $^{2}$McGill University}\\
\makebox[\dimexpr\textwidth-2\tabcolsep\relax][c]{%
$^{3}$New York University \qquad $^{4}$Concordia University}%
}
\iclrfinalcopy
\begin{document}
\maketitle
\lhead{Preprint}

\begin{abstract}
Neural surrogates have emerged as fast alternatives to the numerical simulation of three-dimensional turbulence.
However, training them at high resolution remains challenging, since the memory of full-field models grows with the resolution.
In addition, full-resolution training data are expensive to simulate and store, and therefore scarce.
We introduce \emph{ScaleSplit-NO} (Scale-Split Neural Operator), which exploits the scale structure of turbulence with two neural operators: a Parent predicts the global coarse field at the next time step, and a Child predicts full-resolution local patches conditioned on this prediction.
Neither model operates on the full-resolution field.
The Child is pretrained alone and then attached to the Parent's coarse prediction through zero-initialized connections.
On two complex high-resolution turbulence benchmarks, ScaleSplit-NO surpasses all competing baselines in both prediction accuracy and data efficiency.
On the higher-resolution dataset JHTDB256 ($256^3$), its normalized mean squared error (NMSE) is 53\% lower than that of the strongest baseline, and its training memory is 79\% lower than that of the most memory-efficient baseline.
We further demonstrate its effectiveness for urban wind prediction in a real district of Montreal on a $500\times150\times500$ grid, reducing one-step NMSE by 65.8\% relative to the baseline. 
Moreover, swapping in a Parent trained on additional coarse fields improves prediction without retraining the Child, providing further accuracy gains at a small storage cost.
\end{abstract}

\input{sections/introduction}

\input{sections/related_work}
\input{sections/method}

\input{sections/experiments}
\input{sections/conclusion}

\bibliography{iclr2027_conference}
\bibliographystyle{iclr2027_conference}

\appendix
\clearpage
\startcontents[appendixoutline]
\begingroup
\color{black}
\hypersetup{linkcolor=black,citecolor=black,urlcolor=black,pdfborder={0 0 0}}
\section*{Appendix}
Appendix~\ref{app:related} reviews the broader literature on neural operators. Appendix~\ref{app:results} presents prediction visualizations and additional results. Appendix~\ref{app:protocol} describes the datasets and evaluation protocol, and Appendix~\ref{app:implementation} details the architectures, training procedures, baselines, and resource measurements. Appendix~\ref{app:theory} gives the supplementary theory for initialization, patch assembly, error decomposition, and Parent replacement.

\subsection*{Appendix Contents}
\fontsize{8.95}{17.5}\selectfont
\setlength{\parskip}{0pt}
\printcontents[appendixoutline]{}{1}{\setcounter{tocdepth}{2}}
\endgroup
\clearpage
\input{sections/appendix/visualizations}
\input{sections/appendix/additional_results}

\input{sections/appendix/datasets}

\input{sections/appendix/implementation}

\input{sections/appendix/theory}

\stopcontents[appendixoutline]
\end{document}

%% file: sections/introduction.tex
\section{Introduction}
\label{sec:introduction}\label{sec:intro}
Numerical simulation has long been the primary tool of turbulence research \citep{rogallo1984numerical}, and a central task of turbulence research is to compute how a turbulent field evolves from a given state.
In three dimensions, this is an expensive computation: a single high-resolution trajectory can take hours to days to simulate \citep{moin1998direct}, and every new initial state requires a simulation of its own.
Recent advances in neural operators \citep{lu2021learning,FNO,kovachki2023neural} have made learned surrogates a promising alternative, as they learn the evolution from existing trajectories and return a prediction in a single forward pass.
Training such surrogates at high resolution, however, remains a significant challenge.
A model that processes the whole field stores activations for every grid point, and recent 3D models at $96^3$ to $128^3$ report training on four A100 GPUs for up to a day \citep{P3d,EddyFormer}.
Moreover, the dimension of the field to be predicted far exceeds the number of available samples: a single high-resolution field contains tens of millions of values, and common datasets contain only hundreds to thousands of such fields \citep{li2008public,takamoto2022pdebench,the_well}.
Generating more is costly not only in simulation time but also in storage and transfer.

Existing surrogates are trained end to end, and a single model predicts all scales at once from the full-resolution field \citep{ronneberger2015u,FNO,MG-TFNO,ReViT,EddyFormer,P3d}.
They therefore face both costs, in memory and in data.
Models that take the whole field as input obtain a single sample from each field \citep{ronneberger2015u,FNO,ReViT,EddyFormer}.
When the computation is split into patches that are processed in parallel, the activation memory is distributed over several GPUs, but every update still requires the whole field \citep{MG-TFNO}.
P3D recovers the information outside a patch with a global context network trained on the whole domain \citep{P3d}.
A crop-only variant of P3D, without the context network, still needs large crops, since the accuracy drops as the crops become smaller \citep{P3d}.
Each of these approaches thus either takes the full-resolution field as input at a high memory cost or loses accuracy without the full-resolution field (Figure~\ref{fig:accmem}).

\begin{wrapfigure}{r}{0.6\linewidth}
\vspace{0pt}
\centering
\includegraphics[width=\linewidth]{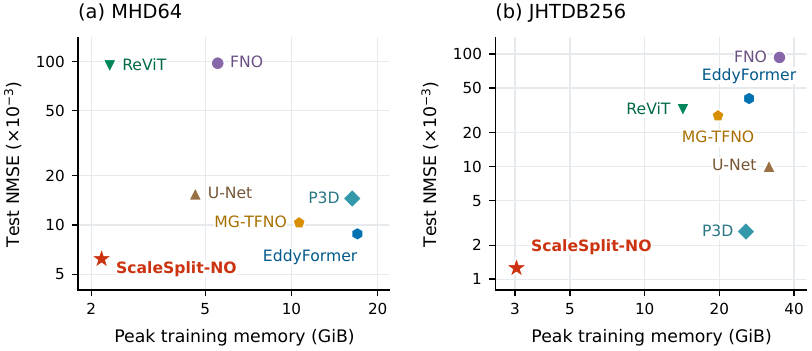}\par
\vspace{-10pt}
\caption{\textbf{Accuracy versus training memory.}
Test error (NMSE) with limited training data against the peak GPU memory during training (Table~\ref{tab:main}).
Lower left is better.}
\label{fig:accmem}
\end{wrapfigure}
In this work, we exploit the scale structure of turbulence to avoid this trade-off.
Turbulent motion contains structures of many sizes.
Large structures are correlated across broad regions of the domain, and small structures only over short distances \citep{pope2000turbulent}.
The small scales at one location are shaped by the large-scale motion around them and only weakly by the small scales elsewhere in the domain \citep{tennekes1975eulerian}.
Predicting the small scales in a patch therefore mainly requires the patch itself and the large-scale motion around the patch.
This large-scale motion varies slowly in space and can be represented on a much coarser grid.
The large scales can hence be learned on a coarse grid, and the small scales on local patches, so that no model needs the whole field at full resolution.

To this end, we propose ScaleSplit-NO (Scale-Split Neural Operator), which consists of two separately supervised neural operators.
The Parent is a neural operator on the coarse grid and predicts the coarse field of the next time step.
The Child is a deeper and wider neural operator shared across all patch positions and predicts the complete patch of the next time step from the input patch and a condition.
This condition is the Parent's coarse prediction, interpolated to the fine grid and cropped at the patch location, and it supplies the large-scale motion that the patch itself does not contain.
Coarsening removes small scales that also affect the evolution of the large scales \citep{leonard1975energy}, so the coarse prediction carries errors.
The Child outputs the complete patch and can therefore correct these errors.
The overlapping patch predictions are finally assembled into the full field.
The two models are trained separately, and the Child is trained in two stages.
The Parent is trained on coarse fields, and the Child is pretrained on patches without a condition.
The Child is then fine-tuned with the condition attached through zero-initialized connections \citep{ControlNet}, so that fine-tuning starts from the pretrained Child.
At inference, the condition comes from the coarse prediction of the Parent, which carries errors.
During fine-tuning, the condition is therefore also computed from the coarse predictions of the frozen Parent and not from the coarse target.

In addition to reducing the memory, the split makes better use of limited training data.
Under spatial homogeneity \citep{batchelor1953theory}, every patch position provides an example of the same local dynamics, so a single frame already yields many examples for the Child.
Each frame, in contrast, contains only one example of the large-scale motion for the Parent.
Each model can therefore be given a capacity and a training schedule that match its data.
Restricting the Child to its patch and the coarse field builds the locality of the small scales into the model as an inductive bias.
A further benefit of separate training is that it allows the Parent to be trained on additional coarse fields alone.
Since coarse fields take only a small fraction of the storage of full-resolution fields, additional coarse fields address the shortage of data for the Parent efficiently.

Our contributions are summarized as follows:
\begin{itemize}
\item We propose ScaleSplit-NO, which exploits the multiscale structure of turbulence by coupling global coarse forecasts with local full-resolution predictions, so that neither of its two neural operators processes the full-resolution field and its training memory does not grow with the resolution.
The scale split also serves as an inductive bias, so ScaleSplit-NO makes better use of limited training data.
\item On two complex high-resolution turbulence benchmarks, ScaleSplit-NO surpasses all competing baselines in prediction accuracy under both low-data and full-data settings.
At $256^3$, it lowers the error by 53\% relative to the strongest baseline and needs 79\% less training memory than the most memory-efficient baseline.
\item Separate training offers a storage-efficient way to improve accuracy.
The Parent can be retrained on additional coarse fields without retraining the Child.
For the same storage, these coarse fields improve the accuracy of ScaleSplit-NO more than full-resolution fields.
\end{itemize}

%% file: sections/related_work.tex
\section{Related work}
\label{sec:related}
\textbf{Neural surrogates for turbulent flows.}
Machine learning has been used to assist turbulence solvers, for instance through learned corrections that allow accurate simulation on coarser grids \citep{kochkov2021machine} or improve spectral solvers at fixed resolution \citep{DBLP:journals/tmlr/DresdnerKNZSBH23}, and to replace them by learning the flow evolution directly from data.
Such surrogates have been built from convolutional networks \citep{wang2020towards,stachenfeld2021learned}, recurrent networks acting on compressed latent states \citep{nakamura2021convolutional}, and neural operators such as the Fourier neural operator (FNO) \citep{FNO}.
Notably, \citet{stachenfeld2021learned} show that a convolutional simulator can be more accurate than a numerical solver at the same coarse resolution.
The wide range of scales in turbulence has since become the focus of more recent models.

\textbf{Multiscale modeling.}
Several neural operators add multiscale or local structure, through multiwavelet bases \citep{gupta2021multiwavelet}, U-Net pathways alongside Fourier layers \citep{wen2022u,li2023long}, or local differential and integral kernels \citep{local-kernel}.
LOGLO-FNO adds a parallel branch of local spectral convolutions and a high-frequency propagation module to recover the high frequencies that global Fourier layers tend to miss \citep{LOGLO-FNO}.
EddyFormer, designed for three-dimensional turbulence, splits the flow into a large-scale stream with global attention and a subgrid-scale stream with local convolutions \citep{EddyFormer}.
\citet{MG-TFNO} process patches of the field in parallel, each with progressively coarser global inputs, and P3D inserts a global context model between the local representations of its patches to include long-range dependencies \citep{P3d}.
However, these models capture multiscale features by training end-to-end on full-resolution fields, which limits them when the turbulence is resolved at very high resolution.
ScaleSplit-NO instead supervises each scale separately, so that no model has to process the entire domain at full resolution during training.

%% file: sections/method.tex
\section{Method}
\label{sec:method}
Our goal is to train accurate surrogates of turbulent fields with a training memory that does not grow with the resolution and with limited training data.
Motivated by the scale structure of turbulence, we split the prediction between two models that operate on different spatial scales: a Parent on a coarse grid of the whole domain and a Child on full-resolution patches.
Consequently, no activations for the full-resolution field need to be stored, and each pair of frames provides many training patches.
Turbulent motion spans a continuous range of scales, from structures as large as the domain down to eddies a few grid cells wide \citep{pope2000turbulent}.
The large scales vary slowly in space and are well represented on a coarse grid, and the small scales require the full resolution but depend mainly on a local neighborhood.

The Parent predicts the coarse field of the next time step from that of the current one.
The Child predicts the complete patch of the next time step from the patch of the current one and from the Parent's coarse prediction cropped to the same location, which we call the condition.
Overlapping patch predictions are then assembled into the predicted field (Figure~\ref{fig:pipeline}).

\begin{figure}[t]
\centering
\includegraphics[width=\linewidth]{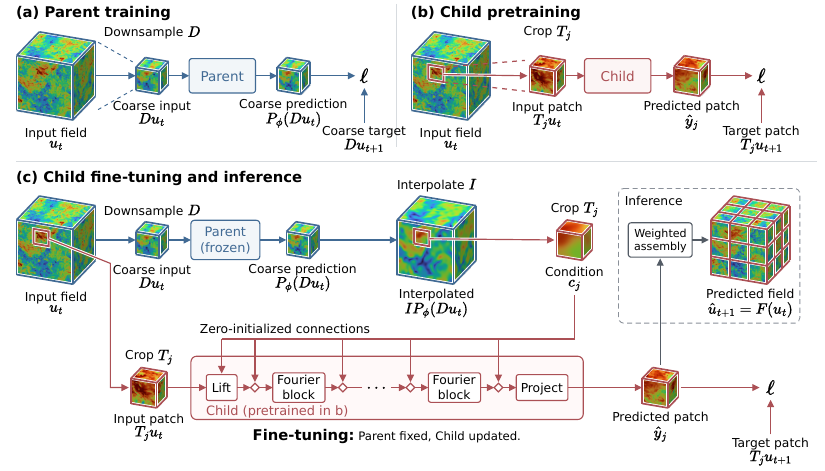}
\caption{\textbf{Separate supervision and composed prediction.}
(a) The Parent is trained to predict the coarse target $Du_{t+1}$ from the coarse input $Du_t$.
(b) The Child is pretrained to predict the target patch $T_ju_{t+1}$ from the input patch $T_ju_t$.
(c) The coarse prediction of the frozen Parent is interpolated and cropped to the condition $c_j$, which enters the Child through zero-initialized connections (diamonds).
Fine-tuning updates the backbone and these connections, and the loss $\ell$ is the only place where the target field enters.
At inference all weights are fixed, and the predicted patches $\hat y_j$ are assembled into the predicted field $\hat u_{t+1}=F(u_t)$.}
\label{fig:pipeline}
\end{figure}

\subsection{Problem formulation}
\label{sec:setup}\label{sec:formulation}\label{sec:task}
We consider a field $u_t\in\mathbb R^{q\times N^3}$ with $q$ physical channels on an $N^3$ grid of points (voxels) and its state $u_{t+1}$ one time interval later.
A simulated sequence of such fields is a trajectory, and each field in it is a frame.
Throughout, we employ two grids of different resolution.
A coarse grid of $R^3$ points, with $R<N$, carries the large scales, that is, the low spatial frequencies of the field, and $R$ is set by the highest frequency assigned to them.
Patches of $p^3$ points of the fine grid carry the small scales.
We connect the two grids through three linear operators.
$D$ truncates a fine field to its low frequencies and samples the result on the coarse grid, $I$ performs the inverse interpolation back to the fine grid, and $T_j$ extracts the patch at position $j$.
On a periodic domain the patches wrap around the boundary, and otherwise they are placed inside the domain.
We call $u_t$ the input field and $u_{t+1}$ the target field, $Du_t$ and $Du_{t+1}$ the coarse input and the coarse target, and $T_ju_t$ and $T_ju_{t+1}$ the input patch and the target patch.
The corresponding outputs of the models are the predicted field, the coarse prediction, and the predicted patch.
Our method is built from three maps: a Parent $P$ that predicts the coarse target, a Child $C$ that predicts the target patch, and an assembly $A$ that combines patch predictions into a field.
In this work, both the Parent and the Child are implemented as Fourier neural operators (FNOs) \citep{FNO}.
Their composition is a predictor $F$ from $u_t$ to $u_{t+1}$, evaluated by the normalized mean squared error over the full field (Section~\ref{sec:protocol}).
For multi-step prediction, the output is fed back as the next input.

\subsection{Coarse prediction as the condition}
\label{sec:composition}
The evolution of a patch depends on the large-scale motion around it, and without this context, the accuracy of patch-based models decreases as the patches become smaller.
Existing models either process the whole field during training or lose accuracy on crops alone.
We instead take the context from the Parent's coarse prediction, cropped to the location of the patch.
The coarse prediction is supervised directly, which allows the Parent to be trained and evaluated independently of the Child.
With sufficient training data, conditioning on the coarse input instead of the coarse prediction yields a higher error (Section~\ref{sec:ablation}).

The Parent network $G_\phi$ predicts the normalized change $Du_{t+1}-Du_t$ between frames from $Du_t$, and the change is rescaled and added back to the input,
\begin{equation}
 P_\phi(z)=z+\mathsf N_{\Delta}^{-1}\!\big(G_\phi(\mathsf N_z z)\big),\qquad z=Du_t,
 \label{eq:parent-decoding-main}
\end{equation}
with fixed normalization maps $\mathsf N_z$ and $\mathsf N_\Delta$, so that its output is a coarse field in physical units.
When whole coarse fields are too few to learn the large-scale dynamics, the Parent can also be trained on patches of the coarse grid, which gives it more training data from the same frames.
The predicted increments of these patches are assembled and then added to $Du_t$ (Appendix~\ref{app:physical-maps}).
On JHTDB256, the Parent is trained on patches of the coarse grid (Appendix~\ref{app:ablations-results}).
Each of these patches is still eight times as wide as a Child patch along each axis and supplies context that the Child does not see.

\subsection{Complete-patch prediction from an imperfect condition}
\label{sec:local-reconstruction}
The Child takes the input patch and the condition as inputs and predicts the complete target patch.
The Child treats the condition as one input among others and does not add a residual to the condition, since the coarse prediction is inexact and the fine-scale information in the patch lets the Child correct the coarse prediction.
This inexactness arises in part from the energy transfer between scales in turbulence \citep{leonard1975energy}.
The small structures removed by coarsening also affect the evolution of the large ones, and the coarse field alone therefore need not determine its next coarse state.
A model that observes only the coarse field then cannot resolve this ambiguity, however much coarse data it is trained on.

Inspired by ControlNet \citep{ControlNet}, we first pretrain the Child on patches alone, without a condition, and then fine-tune it with the condition attached through zero-initialized weights.
Fine-tuning thus starts from the pretrained Child and then updates both its backbone and the condition pathway.
The condition is itself an imperfect input of the same size as the patch.
A Child trained with the condition from the start attains a higher error (Section~\ref{sec:ablation}), possibly because the condition doubles the already redundant input channels and the Child may then learn correlations with the errors of the coarse prediction.
During fine-tuning, the condition is computed from the coarse predictions of the frozen Parent and not from the coarse target, so that the Child is trained with the same kind of condition as at inference.

\textbf{Prediction and condition pathway.}
For patch $j$, the condition and the prediction are
\begin{equation}
 c_j=T_jIP_\phi(Du_t),\qquad
 \widehat y_j=C_{\theta,\psi}(T_ju_t,\,c_j),
 \label{eq:composite-map}
\end{equation}
where $C_{\theta,\psi}$ is the Child, with backbone parameters $\theta$ shared across all patch positions, and $\psi$ are the parameters of the condition pathway.
We write $\bar x$ and $\bar c$ for the patch and the condition after normalization with training-set statistics.
The condition enters at the lifting layer and, through a pointwise linear map, after each Fourier block,
\begin{equation}
 h_0=A_x\bar x+A_c\bar c+b,\qquad
 h_\ell\leftarrow h_\ell+\beta_\ell(\bar c),
 \label{eq:lift}
\end{equation}
where $A_c$ and the pointwise maps $\beta_\ell$ belong to $\psi$.
With $\psi=0$ the pathway is inactive, and $C_{\theta,0}(x,c)=C^{\mathrm{pre}}_{\theta}(x)$ for every patch and every condition (Lemma~\ref{lem:zero}).

\textbf{Assembly.}
Because patches are predicted independently, their predictions disagree along shared boundaries.
We average overlapping predictions with Hann windows $w_j$ \citep{hann_window}, smooth weights that are largest at the patch center and decay towards zero at its faces,
\begin{equation}
 F(u_t)(x)=\sum_j a_j(x)\,\widehat y_j(x),\qquad
 a_j(x)=\frac{w_j(x)}{\sum_k w_k(x)},
 \label{eq:assembly}
\end{equation}
so that every voxel is dominated by the patches that contain it in their interior.
Since the weights sum to one, exact patch predictions assemble to the exact field (coverage and overlap in Appendix~\ref{app:implementation}).

\subsection{Separate supervision}
\label{sec:training}\label{sec:conditioning}\label{sec:objectives}
We train the Parent and the Child separately, in three stages: Parent training, Child pretraining, and Child fine-tuning.
Training them separately allows the capacity and the training schedule of each model to be matched to the amount of training data available to it.
Each frame of a trajectory gives the Child one training patch per patch position, thousands per trajectory, but gives the Parent only one coarse field, and consecutive frames are highly correlated.
The Child accordingly has twice the depth and width of the Parent and is trained for more epochs.
Section~\ref{sec:ablation} compares this configuration with joint training and with other Parent grids and Child patch sizes.
The Parent is trained with the 48 rotations and reflections of the cube as augmentation against overfitting to its few coarse fields, which is inexpensive on the coarse grid.
The Child is trained without augmentation in both pretraining and fine-tuning, as are all baselines.
Even so, with few training frames the Parent has few coarse fields to learn from, and the prediction improves when the Parent is trained on more of them (Section~\ref{sec:swap}).
Its training data are coarse fields, which take $R^3/N^3$ of the storage of a full field, 5.3\% on MHD64 and 1.6\% on JHTDB256, and coarse fields of additional frames can be stored at a small cost.
Since the Parent enters the Child only through $c_j$, a Parent trained on additional coarse fields can be swapped in without retraining the Child, and Section~\ref{sec:swap} compares this use of storage with storing more full-resolution fields.
All three stages minimize the normalized root mean squared error in normalized units, and Appendix~\ref{app:training} gives the three objectives.

\subsection{Training memory}
\label{sec:deployment}
For a fixed patch size, coarse grid, and batch size, the training memory of ScaleSplit-NO does not grow with the resolution of the field.
The Child processes $p^3$ voxels per patch and the Parent at most $R^3$ voxels per sample, and each stage is trained on its own.
For $N=256$, $p=16$, and a batch of 16 patches, one batch holds $1/256$ of a field.
Section~\ref{sec:accuracy} reports the measured training memory of all methods.

%% file: sections/experiments.tex
\section{Experiments}
\label{sec:experiments}\label{sec:protocol}
\textbf{Datasets.}
We consider two high-fidelity 3D turbulence datasets with complex multiscale features.
MHD64 is the magnetohydrodynamic turbulence of The Well \citep{the_well} with a spatial resolution of $64^3$.
Its input is more complex than a single velocity field, consisting of coupled density, velocity, and magnetic fields that change strongly from one frame to the next.
We train on eight trajectories of 100 frames and test on a new trajectory that starts from an initial condition not seen in training.
JHTDB256 is the forced isotropic turbulence of the Johns Hopkins Turbulence Database \citep{perlman2007data,li2008public}, generated by a high-accuracy direct numerical simulation, which we use at $256^3$.
At this resolution, the flow carries rich small-scale detail, and all baselines use relatively small configurations to be trained on a single 40 GB GPU.
We train on 400 frames of one sequence and test on later frames of the same sequence, so that the model predicts the future of the flow.
The low-data setting uses 1/8 of the training data, i.e., a single trajectory, and thus a single initial condition, of MHD64 and 50 frames of JHTDB256 (Appendix~\ref{app:protocol}).

\textbf{Baselines.}
The baselines include the classic U-Net \citep{ronneberger2015u} and FNO \citep{FNO}, as well as the more recent and advanced neural surrogates MG-TFNO \citep{MG-TFNO}, EddyFormer \citep{EddyFormer}, P3D \citep{P3d}, and ReViT \citep{ReViT}.
All of them are trained on the same data as ScaleSplit-NO.
Further details of the baselines are in Appendix~\ref{app:baselines}, and the configuration of ScaleSplit-NO is in Appendix~\ref{app:implementation}.

\textbf{Metric.}
We report the normalized mean squared error
$\mathrm{NMSE}=\frac1q\sum_{k=1}^{q}\frac{\|F(u_t)_k-u_{t+1,k}\|_2^2}{\|u_{t+1,k}\|_2^2}
$, where $k$ runs over the $q$ physical channels.
The one-step NMSE is the average error of a single prediction step.
The rollout NMSE is the average error along autoregressive rollouts, whose length on each dataset is chosen so that the flow loses a similar degree of correlation with its initial state.
Further details of the evaluation are in Appendix~\ref{app:protocol}.

\subsection{Main results}
\label{sec:accuracy}\label{sec:resources}
\textbf{Accuracy.}
Table~\ref{tab:main} summarizes the results.
ScaleSplit-NO has the lowest one-step and rollout error on both datasets, with 1/8 and with all of the training data.
On MHD64 its one-step error is 23\% lower than that of the strongest baseline, EddyFormer, with all eight trajectories and 30\% lower with a single one.
On JHTDB256 it is 53\% lower than that of the strongest baseline, P3D, with 1/8 as well as with all of the data, and its rollout error is 10\% and 35\% lower.
We believe this gain comes mainly from the smaller patches of ScaleSplit-NO, which yield many more training samples from the same data.
P3D also trains on crops, but uses large ones of $128^3$, since smaller crops lose more context \citep{P3d}.
Since the Parent provides the large-scale context, the Child can be trained on $16^3$ patches, and a field yields 512 times as many of these as $128^3$ crops.
Figure~\ref{fig:vis3d} shows one test pair from each dataset.
The predictions of ScaleSplit-NO reproduce the structures of the ground truth, including the wide range of scales in JHTDB256, and its errors are visibly smaller than those of the strongest baseline (energy spectra in Appendix~\ref{sec:diagnostics}).

\begin{table}[t]
\newcommand{\firstb}[1]{{\def\ULthickness{0.9pt}\uline{#1}}}\newcommand{\secondb}[1]{\setbox0=\hbox{#1}\rlap{\raisebox{-0.45ex}[0pt][0pt]{\hbox to \wd0{\rule{0.12em}{0.5pt}\xleaders\hbox{\kern0.08em\rule{0.12em}{0.5pt}}\hfill}}}\box0}
\definecolor{better}{HTML}{16366E}\definecolor{worse}{HTML}{8C1616}
\caption{\textbf{Test error (NMSE) and training memory.}
U-Net serves as the reference.
Parentheses give the change in error and memory relative to U-Net, \mbox{\protect\textcolor{better}{($\downarrow$ better)}} or \mbox{\protect\textcolor{worse}{($\uparrow$ worse)}}.
\textbf{Bold}: best result.
\protect\firstb{Thick underline}: best baseline.
\protect\secondb{Dashed underline}: second-best baseline.}
\label{tab:main}\label{tab:main-jhtdb}
\begin{center}\small\setlength{\tabcolsep}{3pt}\renewcommand{\arraystretch}{1.08}
\begin{tabular*}{\textwidth}{@{\extracolsep{\fill}}cl r@{\extracolsep{0pt}\,}l@{\extracolsep{\fill}\hspace{3pt}} r@{\extracolsep{0pt}\,}l@{\extracolsep{\fill}\hspace{3pt}} r@{\extracolsep{0pt}\,}l@{\extracolsep{\fill}\hspace{3pt}} r@{\extracolsep{0pt}\,}l@{\extracolsep{\fill}\hspace{3pt}} r@{\extracolsep{0pt}\,}l@{}}
\toprule
& & \multicolumn{4}{c}{NMSE, low data (1/8)} & \multicolumn{4}{c}{NMSE, full data} & \multicolumn{2}{c}{Training}\\
\cmidrule(lr){3-6}\cmidrule(lr){7-10}\cmidrule(l){11-12}
& & \multicolumn{2}{c}{One step} & \multicolumn{2}{c}{Rollout} & \multicolumn{2}{c}{One step} & \multicolumn{2}{c}{Rollout} & \multicolumn{2}{c}{Memory}\\
& Method & \multicolumn{2}{c}{$\times10^{-3}$} & \multicolumn{2}{c}{$\times10^{-2}$} & \multicolumn{2}{c}{$\times10^{-3}$} & \multicolumn{2}{c}{$\times10^{-2}$} & \multicolumn{2}{c}{GiB}\\
\midrule
\multirow{7}{*}{\rotatebox[origin=c]{90}{MHD64}} & U-Net & 15.5 &  & 2.83 &  & \secondb{5.97} &  & 1.30 &  & \secondb{4.6} & \\
 & FNO & 97.6 & \textcolor{worse}{($\uparrow$529\%)} & 17.2 & \textcolor{worse}{($\uparrow$507\%)} & 29.5 & \textcolor{worse}{($\uparrow$394\%)} & 4.96 & \textcolor{worse}{($\uparrow$283\%)} & 5.5 & \textcolor{worse}{($\uparrow$20\%)}\\
 & MG-TFNO & \secondb{10.3} & \textcolor{better}{($\downarrow$33\%)} & \secondb{2.06} & \textcolor{better}{($\downarrow$27\%)} & 5.98 & (0\%) & \secondb{1.27} & \textcolor{better}{($\downarrow$2\%)} & 10.6 & \textcolor{worse}{($\uparrow$130\%)}\\
 & EddyFormer & \firstb{8.83} & \textcolor{better}{($\downarrow$43\%)} & \firstb{1.99} & \textcolor{better}{($\downarrow$30\%)} & \firstb{5.88} & \textcolor{better}{($\downarrow$2\%)} & 1.32 & \textcolor{worse}{($\uparrow$2\%)} & 17.0 & \textcolor{worse}{($\uparrow$268\%)}\\
 & P3D & 14.5 & \textcolor{better}{($\downarrow$6\%)} & 2.60 & \textcolor{better}{($\downarrow$8\%)} & 6.00 & (0\%) & \firstb{1.23} & \textcolor{better}{($\downarrow$5\%)} & 16.3 & \textcolor{worse}{($\uparrow$253\%)}\\
 & ReViT & 93.8 & \textcolor{worse}{($\uparrow$505\%)} & 16.7 & \textcolor{worse}{($\uparrow$490\%)} & 58.5 & \textcolor{worse}{($\uparrow$880\%)} & 10.1 & \textcolor{worse}{($\uparrow$679\%)} & \firstb{2.3} & \textcolor{better}{($\downarrow$50\%)}\\
 & \textbf{ScaleSplit-NO} & \textbf{6.20} & \textcolor{better}{($\downarrow$60\%)} & \textbf{1.43} & \textcolor{better}{($\downarrow$49\%)} & \textbf{4.52} & \textcolor{better}{($\downarrow$24\%)} & \textbf{1.05} & \textcolor{better}{($\downarrow$19\%)} & \textbf{2.2} & \textcolor{better}{($\downarrow$53\%)}\\
\midrule
\multirow{7}{*}{\rotatebox[origin=c]{90}{JHTDB256}} & U-Net & \secondb{10.1} &  & 14.0 &  & \secondb{6.77} &  & 8.85 &  & 31.7 & \\
 & FNO & 93.3 & \textcolor{worse}{($\uparrow$826\%)} & 105 & \textcolor{worse}{($\uparrow$655\%)} & 40.6 & \textcolor{worse}{($\uparrow$500\%)} & 43.7 & \textcolor{worse}{($\uparrow$393\%)} & 35.0 & \textcolor{worse}{($\uparrow$10\%)}\\
 & MG-TFNO & 28.3 & \textcolor{worse}{($\uparrow$181\%)} & 50.6 & \textcolor{worse}{($\uparrow$263\%)} & 24.8 & \textcolor{worse}{($\uparrow$267\%)} & 15.8 & \textcolor{worse}{($\uparrow$79\%)} & \secondb{19.8} & \textcolor{better}{($\downarrow$38\%)}\\
 & EddyFormer & 40.3 & \textcolor{worse}{($\uparrow$300\%)} & \secondb{10.3} & \textcolor{better}{($\downarrow$26\%)} & 37.3 & \textcolor{worse}{($\uparrow$452\%)} & \secondb{7.49} & \textcolor{better}{($\downarrow$15\%)} & 26.4 & \textcolor{better}{($\downarrow$17\%)}\\
 & P3D & \firstb{2.65} & \textcolor{better}{($\downarrow$74\%)} & \firstb{1.64} & \textcolor{better}{($\downarrow$88\%)} & \firstb{2.20} & \textcolor{better}{($\downarrow$67\%)} & \firstb{1.48} & \textcolor{better}{($\downarrow$83\%)} & 25.6 & \textcolor{better}{($\downarrow$19\%)}\\
 & ReViT & 32.1 & \textcolor{worse}{($\uparrow$218\%)} & 17.0 & \textcolor{worse}{($\uparrow$21\%)} & 31.7 & \textcolor{worse}{($\uparrow$368\%)} & 15.7 & \textcolor{worse}{($\uparrow$77\%)} & \firstb{14.3} & \textcolor{better}{($\downarrow$55\%)}\\
 & \textbf{ScaleSplit-NO} & \textbf{1.25} & \textcolor{better}{($\downarrow$88\%)} & \textbf{1.47} & \textcolor{better}{($\downarrow$89\%)} & \textbf{1.03} & \textcolor{better}{($\downarrow$85\%)} & \textbf{0.959} & \textcolor{better}{($\downarrow$89\%)} & \textbf{3.0} & \textcolor{better}{($\downarrow$90\%)}\\
\bottomrule
\end{tabular*}
\end{center}
\end{table}

\textbf{Training memory.}
The training memory of ScaleSplit-NO depends on the patch size and the coarse grid, not on the resolution of the field (Section~\ref{sec:deployment}).
On MHD64 it is close to that of the lightest baseline (Figure~\ref{fig:accmem}).
At $256^3$ it is 3.0 GiB, 79\% lower than that of the lightest baseline.
Appendix~\ref{app:resources} gives the training memory at different batch sizes and the training and inference times.

\begin{figure}[t]
\centering
\includegraphics[width=\linewidth]{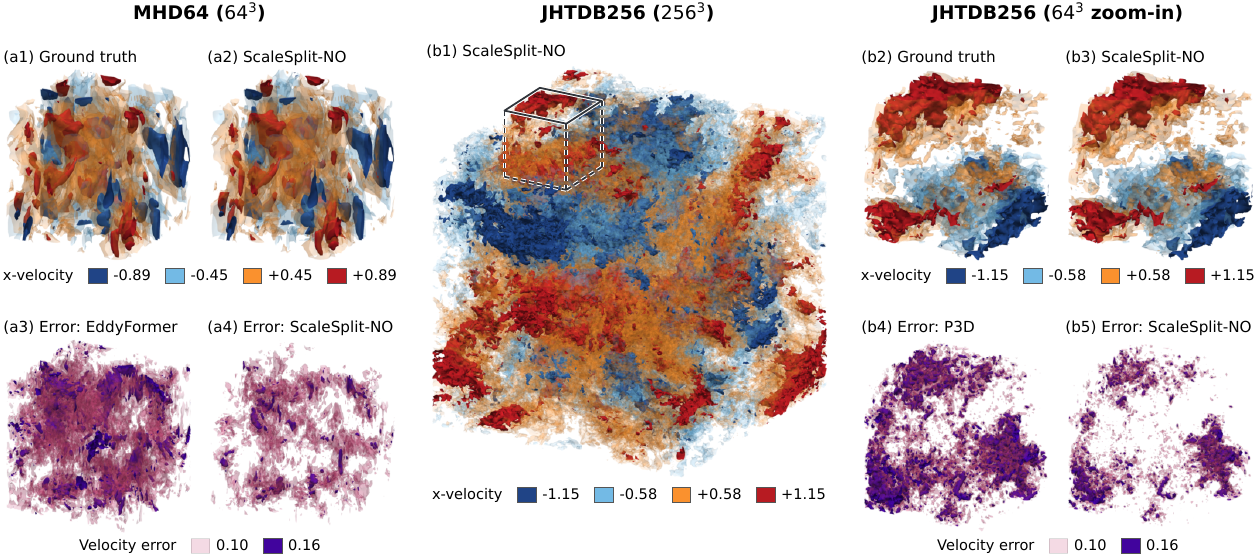}
\caption{\textbf{Visualization of 3D turbulence predictions and errors.}
The figure shows one test pair of each dataset for one-step prediction in the low-data setting.
Each panel draws a 3D field as surfaces of constant value (isosurfaces) at the levels in its legend.
Velocity panels show the $x$-velocity of the ground truth and of the prediction of ScaleSplit-NO.
Error panels show the magnitude of the velocity error \mbox{$\|\hat u_{t+1}-u_{t+1}\|$} of ScaleSplit-NO and of the strongest baseline of each dataset.}
\label{fig:vis3d}
\end{figure}

\textbf{Data efficiency.}
\label{sec:dataeff}\label{sec:swap}%
High-resolution 3D simulation data are expensive to generate and to store, and ScaleSplit-NO needs much less of them (Figure~\ref{fig:dataeff}(a1, a2)).
On MHD64, its error with a single trajectory is within 6\% of that of the best baseline trained on all eight.
On JHTDB256, its error with 1/8 of the data is even 43\% lower than that of the best baseline trained on all of the data.

\begin{figure}[t]
\centering
\includegraphics[width=\linewidth]{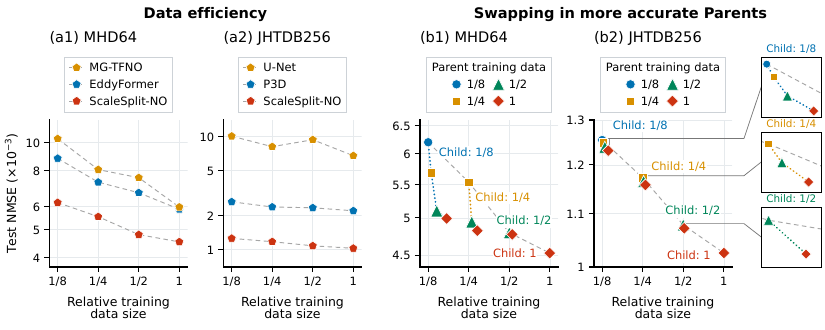}
\caption{\textbf{Test error versus the amount of stored training data.}
(a1, a2) Test error of ScaleSplit-NO and the two strongest baselines of each dataset for different amounts of training data.
(b1, b2) Test error of ScaleSplit-NO when a Child trained on the labeled fraction of the data is kept fixed and combined with Parents trained on additional coarse fields.
In (b1, b2), gray dashed lines connect models whose Parent and Child are trained on the same data, and the relative training data size also counts the additional coarse fields.}
\label{fig:dataeff}\label{fig:swap}
\end{figure}

If only full fields are stored, each gives the Child thousands of patches but the Parent a single coarse field, so the Parent is the one that runs short of data.
Because the two scales are supervised separately, ScaleSplit-NO can store additional coarse fields to train a more accurate Parent.
We train a Parent on the additional coarse fields and combine it with the existing Child at inference, without retraining the Child (Figure~\ref{fig:dataeff}(b1, b2)).
On MHD64, starting from one trajectory, the coarse fields of all eight trajectories add only 4.6\% of the full-data storage and lower the error by 19.5\%, below every baseline trained on all of the data, whereas one more full trajectory adds 12.5\% and lowers the error by only 10.7\%.
On JHTDB256 the gain is small, since its coarse fields are tiny and isotropic turbulence has little large-scale structure, but per unit of storage coarse fields remain more efficient.

\textbf{Real-world application to urban wind.}
\label{sec:cityffd}%
Beyond the two benchmarks, we apply ScaleSplit-NO to a real-world turbulent flow, the wind in an urban district of Montreal, simulated on a $500\times150\times500$ grid \citep{mortezazadeh2022cityffd,qin2025modeling} (Appendix~\ref{app:cityffd}).
Unlike the two benchmarks, the domain is not periodic, and buildings deflect the wind and channel it through the streets.
As a reference, we train U-Net, which is commonly used for urban wind prediction \citep{xiang2021non}.
The one-step NMSE of ScaleSplit-NO is 65.8\% lower than that of U-Net, and its training memory is 2.8 instead of 27.1 GiB.

\subsection{Ablation studies}
\label{sec:ablation}
Unless noted otherwise, all ablations use 1/8 of the MHD64 data and vary the conditioning input of the Child, the training procedure, the sizes of the two models, and the overlap of the patches.
\begin{figure}[t]
\definecolor{worse}{HTML}{8C1616}
\begin{minipage}[t]{0.36\linewidth}
\centering
\raisebox{-\dimexpr\height+4pt\relax}{\includegraphics[width=125pt]{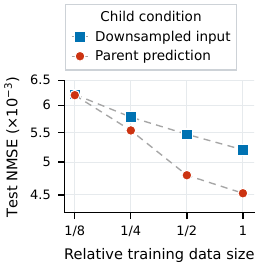}}
\caption{\textbf{Effect of the Parent's coarse prediction.}
Test error on MHD64 with the Child conditioned on the Parent's coarse prediction or on the downsampled input $Du_t$.}
\label{fig:parentneed}
\end{minipage}\hfill
\begin{minipage}[t]{0.605\linewidth}
\centering
\captionof{table}{\textbf{Ablation of ScaleSplit-NO configurations.}
Test error (NMSE) on MHD64 in the low-data setting.
The final configuration serves as the reference.
Parentheses give the change in error relative to the final configuration, \mbox{\protect\textcolor{worse}{($\uparrow$ worse)}}.}
\label{tab:ablation}
\raisebox{-\height}{\small\setlength{\tabcolsep}{4pt}\begin{tabular}{lr@{\,}l}
\toprule
Configuration & \multicolumn{2}{c}{NMSE ($\times10^{-3}$)}\\
\midrule
ScaleSplit-NO final configuration & 6.20 & \\
\midrule
w/o pretraining and zero init. & 6.50 & \textcolor{worse}{($\uparrow$5.0\%)}\\
w/o Child fine-tuning & 6.75 & \textcolor{worse}{($\uparrow$8.9\%)}\\
Fine-tuning with coarse target & 10.9 & \textcolor{worse}{($\uparrow$75.2\%)}\\
Residual correction & 6.68 & \textcolor{worse}{($\uparrow$7.9\%)}\\
Joint Parent--Child training & 7.64 & \textcolor{worse}{($\uparrow$23.2\%)}\\
\midrule
Global coarse grid $24^3\rightarrow16^3$ & 6.20 & \textcolor{worse}{($\uparrow$0.1\%)}\\
Global coarse grid $24^3\rightarrow32^3$ & 6.42 & \textcolor{worse}{($\uparrow$3.5\%)}\\
Local Child grid $16^3\rightarrow12^3$ & 6.30 & \textcolor{worse}{($\uparrow$1.6\%)}\\
Local Child grid $16^3\rightarrow24^3$ & 7.06 & \textcolor{worse}{($\uparrow$13.9\%)}\\
\bottomrule
\end{tabular}}
\end{minipage}
\end{figure}

\textbf{Effect of the Parent's coarse prediction.}
In turbulence, the large scales evolve more slowly than the small ones \citep{tennekes1972first}, but they still change within one time step, and the Parent predicts this change for the Child.
To test the contribution of this prediction, we replace it by the downsampled input $Du_t$ and train the Child in the same way (Figure~\ref{fig:parentneed}, Appendix~\ref{app:training}).
With a single trajectory, the two conditions give nearly the same error.
A likely reason is that 99 pairs of slowly changing coarse fields contain too few distinct changes for the Parent to learn from.
With four and all eight trajectories, the error without the Parent's coarse prediction is 14\% and 15\% higher.
Once the data contain enough changes for the Parent to learn from, its coarse prediction clearly helps.

\textbf{Training procedure.}
Fine-tuning the Child on the coarse target in place of the Parent's predictions raises the error by 75.2\% (Table~\ref{tab:ablation}), presumably because the Child then sees the errors of the coarse prediction only at inference.
Training the Parent and the Child jointly through the Child loss, starting from the same pretrained models, raises the error by 23.2\%, which shows the value of training them separately.
The remaining three variants have a smaller effect: without Child fine-tuning, with a residual correction, or without pretraining and zero initialization, the error rises by 5--9\%.

\textbf{Sizes of the Parent and the Child.}
On MHD64, among the sizes we tested, the final configuration of ScaleSplit-NO, with a global coarse grid of $24^3$ and Child patches of $16^3$, performs relatively well (Table~\ref{tab:ablation}, bottom).
Patches as small as $16^3$ may suffice because the small scales of turbulence are shaped mainly by the nearby flow, so that a small patch together with the Parent's coarse prediction contains most of what the Child needs.
Smaller patches also give more training samples from the same data.
The Parent only has to supply the large-scale motion, for which a coarse grid is enough: a grid of $16^3$ gives the same error as $24^3$, and a finer grid of $32^3$ is slightly worse.
Small grids thus save memory without costing accuracy.
The results on JHTDB256 are given in Appendix~\ref{app:ablations-results}.

\begin{figure}[t]
\begin{minipage}[t]{0.36\linewidth}
\centering
\raisebox{-\dimexpr\height+8pt\relax}{\includegraphics[width=1.68in,trim=0 4.5 0 10,clip]{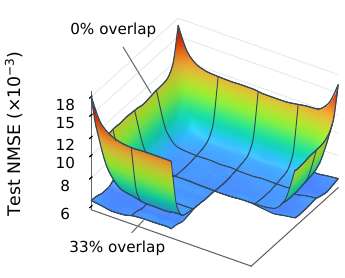}}
\caption{\textbf{Effect of the overlap.}
Spatial distribution of the test error (NMSE) on the central slice of a $16^3$ Child patch on MHD64 in the low-data setting, without overlap and with 33\% overlap.}
\label{fig:edge}
\end{minipage}\hfill
\begin{minipage}[t]{0.605\linewidth}
\centering
\captionof{table}{\textbf{Trade-off between accuracy and inference time for different overlaps.}
Test error (NMSE) in the low-data setting, voxels predicted by the Child relative to the field, and inference time per field.
\textbf{Bold}: default overlap of ScaleSplit-NO in this work.}
\label{tab:overlap}
\raisebox{-\height}{\small\setlength{\tabcolsep}{2.8pt}\begin{tabular}{rccc@{\hskip 5pt}rccc}
\toprule
\multicolumn{4}{c}{MHD64} & \multicolumn{4}{c}{JHTDB256}\\
\cmidrule(r){1-4}\cmidrule(l){5-8}
Overlap & NMSE & Voxels & Time & Overlap & NMSE & Voxels & Time\\
 & $\times10^{-3}$ & & ms & & $\times10^{-3}$ & & s\\
\midrule
0\% & 8.87 & 1.0$\times$ & 26 & 0\% & 2.17 & 1.0$\times$ & 1.35\\
20\% & 6.44 & 2.0$\times$ & 47 & 16\% & 1.33 & 1.7$\times$ & 2.24\\
\textbf{33\%} & 6.20 & 3.4$\times$ & 81 & \textbf{27\%} & 1.25 & 2.6$\times$ & 3.46\\
43\% & 6.10 & 5.4$\times$ & 122 & 38\% & 1.23 & 4.3$\times$ & 5.68\\
50\% & 6.05 & 8.0$\times$ & 182 & 50\% & 1.21 & 8.0$\times$ & 10.6\\
\bottomrule
\end{tabular}}
\end{minipage}
\end{figure}

\textbf{Patch boundaries and overlap.}
The Fourier layers of the Child treat each patch as periodic \citep{FNO}, but a patch cut out of the field is not periodic.
The error of the Child is therefore largest near the faces of the patch, about twice that at the center on the faces and about four times in the corners (Figure~\ref{fig:edge}).
Overlapping the patches reduces this error, because the Hann weights of the assembly are very small near the faces and the larger errors there hardly affect the assembled field.
On MHD64, an overlap of 20\% already lowers the NMSE by about a quarter (Table~\ref{tab:overlap}).
Larger overlaps lower the NMSE by only a few percent more, while the inference time grows quickly and at 50\% is nearly four times that at 20\%.
The overlap therefore trades accuracy against inference time, and we use 33\% by default.
On JHTDB256, the overlap of the Parent patches changes the error by less than 3\% (Appendix~\ref{app:overlap-results}).

\FloatBarrier

%% file: sections/conclusion.tex
\section{Conclusion}
\label{sec:conclusion}
We have presented ScaleSplit-NO, which exploits the multiscale structure of turbulence by coupling global coarse forecasts with local full-resolution predictions. It learns surrogates of high-resolution three-dimensional turbulence with limited GPU memory and training data. Additional coarse fields can further improve prediction accuracy by upgrading the Parent without retraining the Child.

\textbf{Limitations.}
Overlapping patches cover more voxels than the field itself in three dimensions, so inference becomes slower as the overlap grows.
The overlap has to balance accuracy against inference speed.
Moreover, the method is validated only on structured grids, and on unstructured meshes the FNO cannot serve directly as its backbone.

\clearpage

\section*{Acknowledgments}

Shaoxiang Qin, Yucheng Zhao and Xiongye Xiao acknowledge support from the U.S. National Science Foundation (NSF) through the Science and Technology Center for Complex Particle Systems (COMPASS) [Award No. 2243104].
Zongyi Li acknowledges support from Schmidt Sciences AI2050 Fellowship. Liangzhu Leon Wang acknowledges financial support from the Natural Sciences and Engineering Research Council of Canada (NSERC) through the Discovery Grants Program [RGPIN-2024-06297] and the Canada First Research Excellence Fund (CFREF) [IMPACT Project – Transforming Built and Urban Microclimates: Advancing Resilience Science for Vulnerable Populations in a Decarbonized and Electrified Canada].

\label{LastMainPage}
\clearpage

%% file: sections/appendix/visualizations.tex
\section{Related work on neural operators}
\label{app:related}
Operator learning approximates mappings between function spaces, including PDE solution operators \citep{lu2021learning,kovachki2023neural}.
Spectral approaches use Fourier \citep{FNO} or multiwavelet representations \citep{gupta2021multiwavelet}, with extensions through factorized spectral layers \citep{DBLP:conf/iclr/TranMXO23} and direct frequency-domain learning \citep{poli2022transform}.
Related ideas have also been adapted to visual token mixing \citep{DBLP:conf/iclr/GuibasMLTAC22}.
To capture spatial structure, neural operators incorporate U-shaped components \citep{wen2022u,DBLP:journals/tmlr/RahmanRA23,li2023long}, local integral and differential operators \citep{local-kernel}, or combinations of local and global information \citep{MG-TFNO,LOGLO-FNO,EddyFormer,P3d}.
Attention offers another approach, connecting operator approximation with Galerkin projection \citep{cao2021choose} and enabling models to handle heterogeneous inputs \citep{hao2023gnot} or aggregate physics-aware tokens \citep{DBLP:conf/icml/WuLW0L24}.
Other attention-based designs use factorization \citep{li2023scalable} or evolve representations in latent space \citep{alkin2024universal}.
Equivariant architectures explicitly encode selected spatial symmetries \citep{DBLP:conf/icml/HelwigZ0KWJ23,ReViT}.
Alongside architectural design, MPP, Poseidon, and PDE-Transformer study pretraining across PDE systems and adaptation to downstream tasks \citep{mccabe2024multiple,herde2024poseidon,DBLP:conf/icml/Holzschuh0KT25}.
For time-dependent problems, iterative denoising has been used to improve rollout accuracy \citep{lippe2023pde}, and generative operators learn conditional future dynamics \citep{koupai2025enma}.
Generative modeling also allows turbulent states to be sampled from geometry and boundary conditions without an initial flow field \citep{lienen2024zero}.

%% file: sections/appendix/additional_results.tex
\section{Additional results}
\label{app:results}\label{app:diagnostics}

\subsection{Main comparison in NRMSE}
Table~\ref{tab:main-nrmse} repeats the one-step comparison of Table~\ref{tab:main} in NRMSE.
The ranking of the methods is the same as in Table~\ref{tab:main} in all four settings.

\begin{table}[!h]
\caption{One-step test NRMSE ($\times10^{-2}$) with 1/8 and with all of the training data.}
\label{tab:main-nrmse}
\begin{center}\small
\begin{tabular}{lcc@{\hskip 14pt}cc}
\toprule
& \multicolumn{2}{c}{MHD64} & \multicolumn{2}{c}{JHTDB256}\\
\cmidrule(lr){2-3}\cmidrule(l){4-5}
Method & Low data (1/8) & Full data & Low data (1/8) & Full data\\
\midrule
U-Net & 11.3 & 7.01 & 9.94 & 8.14\\
FNO & 29.0 & 15.9 & 30.3 & 19.8\\
MG-TFNO & 9.37 & 7.06 & 16.4 & 15.3\\
EddyFormer & 8.58 & 6.94 & 19.5 & 18.8\\
P3D & 11.2 & 7.10 & 5.07 & 4.62\\
ReViT & 28.4 & 22.5 & 17.4 & 17.3\\
ScaleSplit-NO & \textbf{7.13} & \textbf{6.07} & \textbf{3.49} & \textbf{3.14}\\
\bottomrule
\end{tabular}
\end{center}
\end{table}

\subsection{Visualizations}
\label{app:visualizations}
Figures~\ref{fig:app-mhd-methods}--\ref{fig:app-jht-iso} show the first test pair of each dataset for the models trained with 1/8 of the data.
The normalized error is the signed error divided by the root mean square of the target channel over the complete field, and the NMSE given with each panel refers to this test pair.
Figure~\ref{fig:app-urban} shows the first test pair of the urban wind field (Appendix~\ref{app:cityffd}).

\begin{figure}[p]
\centering
\includegraphics[width=\textwidth]{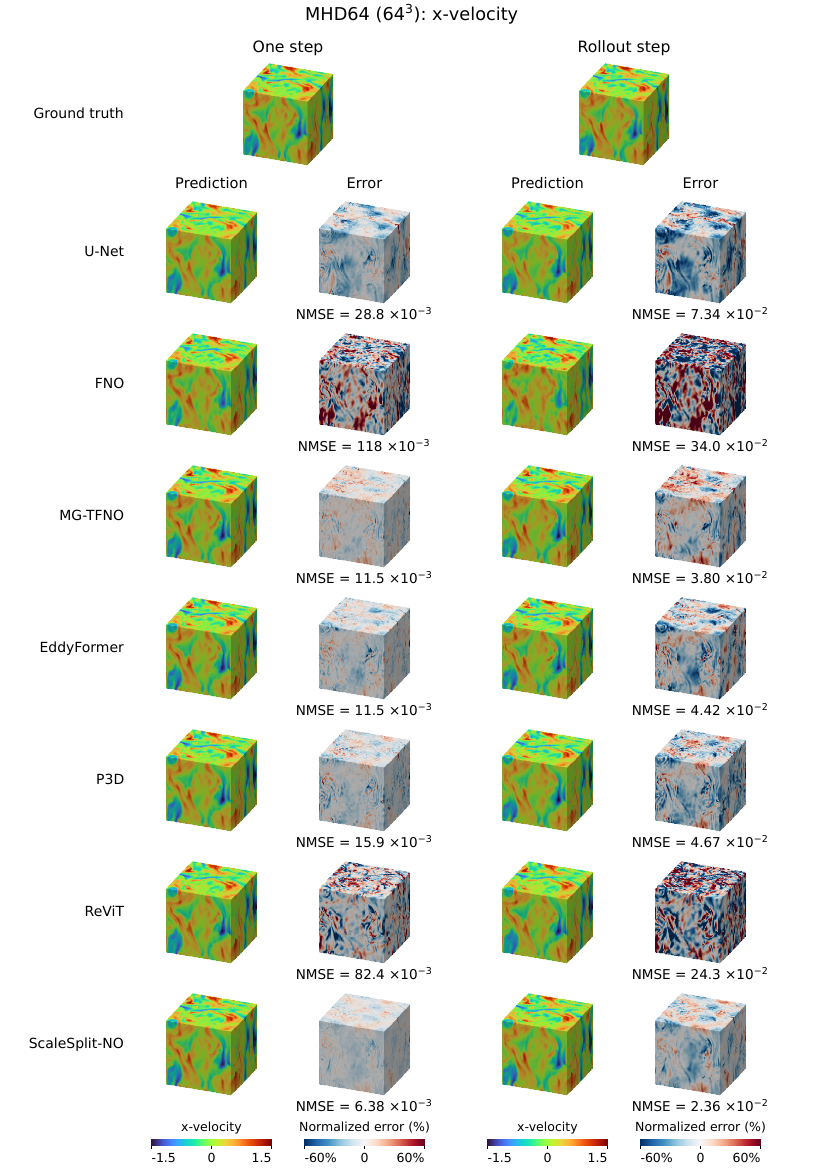}
\caption{\textbf{MHD64, x-velocity of all methods.}
One-step prediction (left) and rollout step 2 (right) with the normalized errors.}
\label{fig:app-mhd-methods}
\end{figure}

\begin{figure}[p]
\centering
\includegraphics[width=\textwidth]{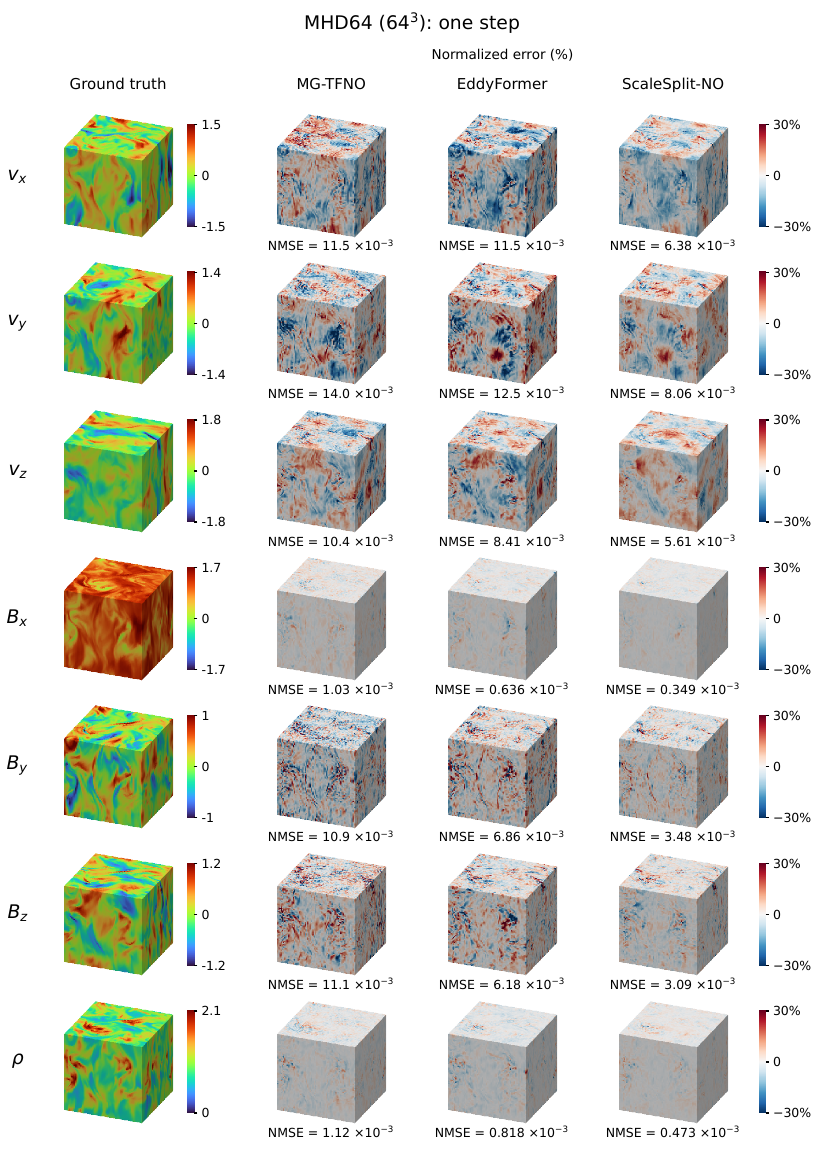}
\caption{\textbf{MHD64, all seven channels at one step.}
Ground truth and normalized errors of the two strongest baselines, MG-TFNO and EddyFormer, and of ScaleSplit-NO.}
\label{fig:app-mhd-channels}
\end{figure}

\begin{figure}[p]
\centering
\includegraphics[width=\textwidth]{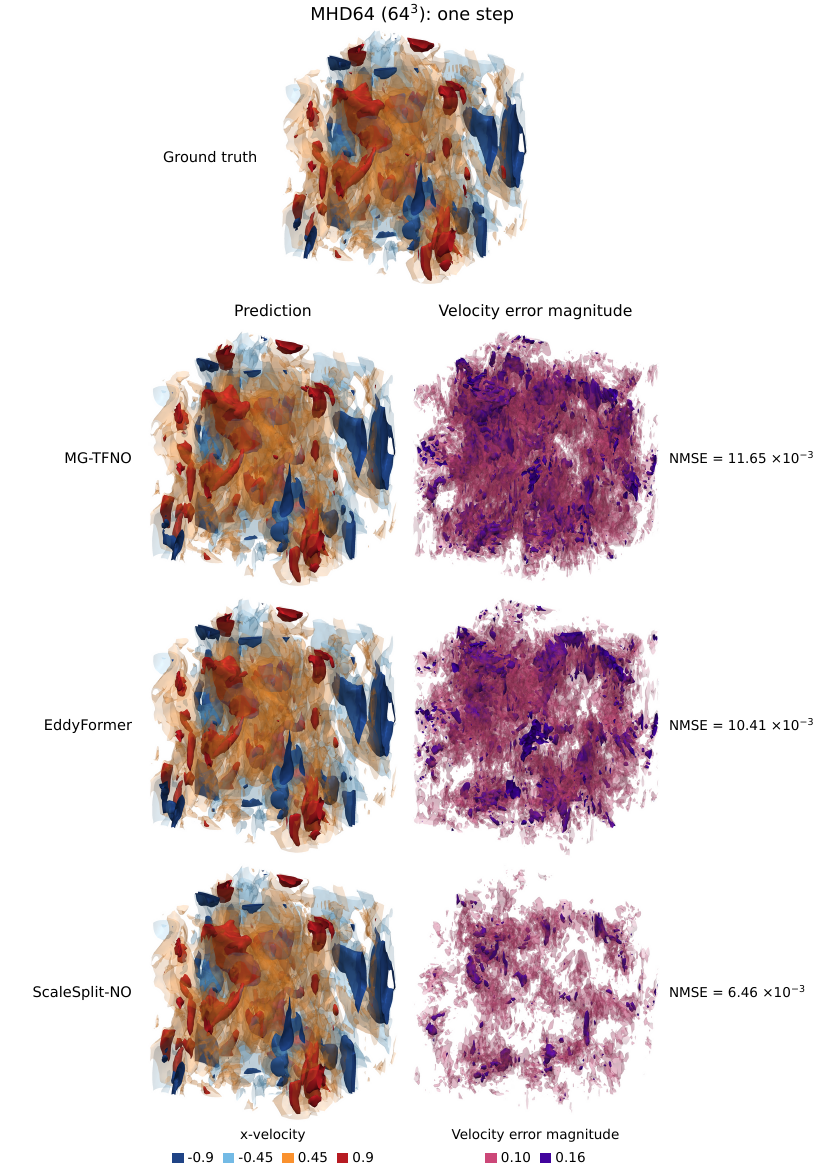}
\caption{\textbf{MHD64, complete field at one step.}
Isosurfaces of the x-velocity at one and two standard deviations of the ground truth, and of the velocity error magnitude at 0.10 and 0.16, for MG-TFNO, EddyFormer, and ScaleSplit-NO.}
\label{fig:app-mhd-iso}
\end{figure}

\begin{figure}[p]
\centering
\includegraphics[width=\textwidth]{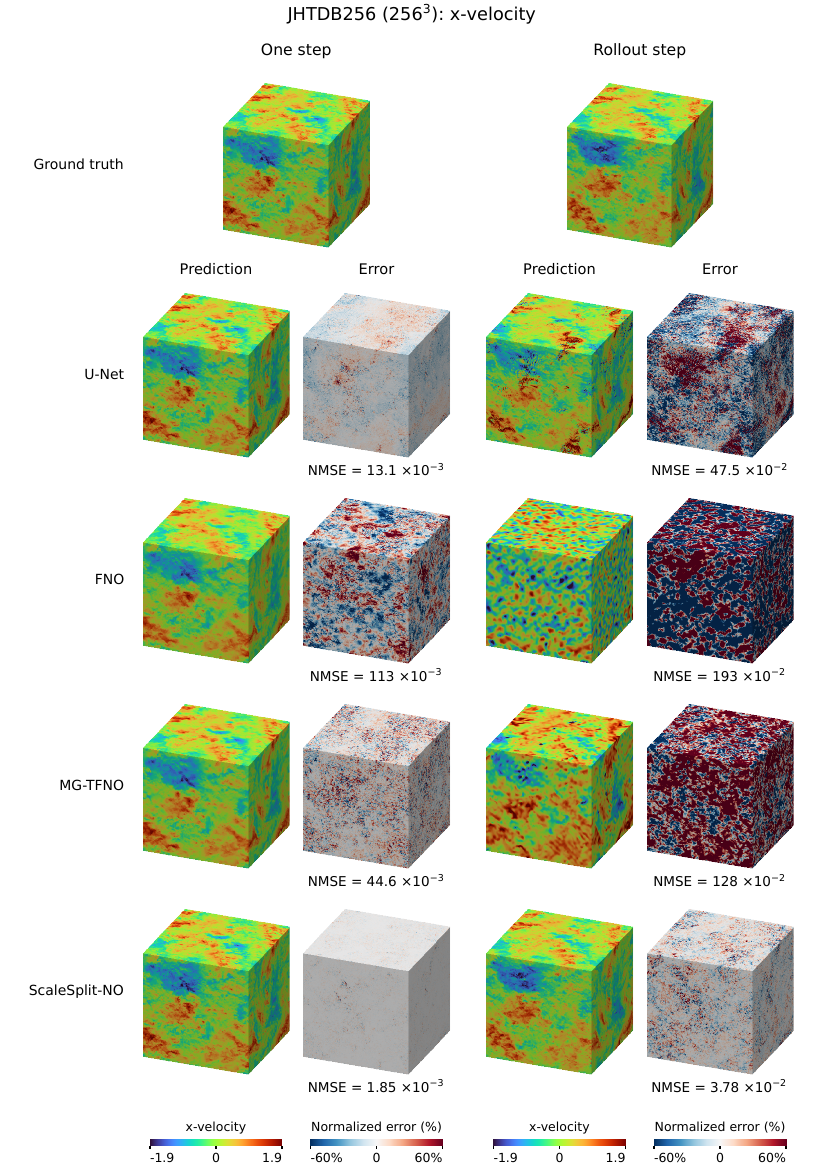}
\caption{\textbf{JHTDB256, x-velocity of U-Net, FNO, MG-TFNO, and ScaleSplit-NO.}
One-step prediction (left) and rollout step 15 (right) with the normalized errors.}
\label{fig:app-jht-methods1}
\end{figure}

\begin{figure}[p]
\centering
\includegraphics[width=\textwidth]{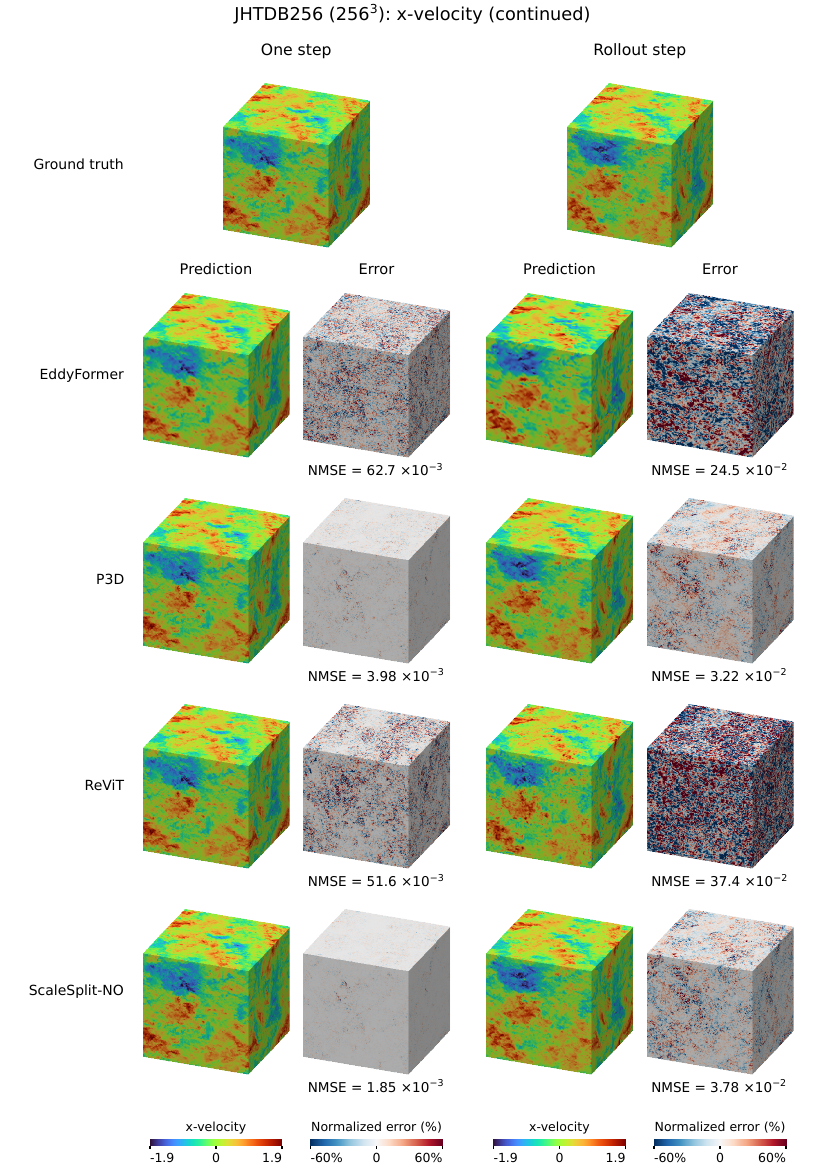}
\caption{\textbf{JHTDB256, x-velocity of EddyFormer, P3D, ReViT, and ScaleSplit-NO.}
One-step prediction (left) and rollout step 15 (right) with the normalized errors, on the same scales as Figure~\ref{fig:app-jht-methods1}.}
\label{fig:app-jht-methods2}
\end{figure}

\begin{figure}[p]
\centering
\includegraphics[width=\textwidth]{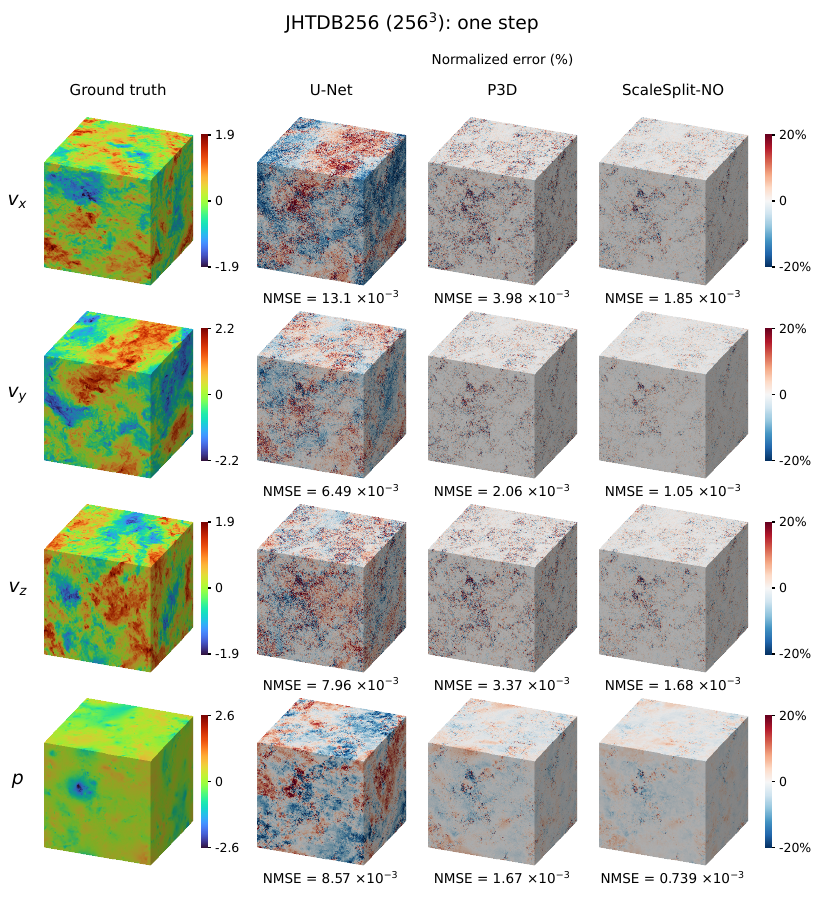}
\caption{\textbf{JHTDB256, all four channels at one step.}
Ground truth and normalized errors of the two strongest baselines, U-Net and P3D, and of ScaleSplit-NO.}
\label{fig:app-jht-channels}
\end{figure}

\begin{figure}[p]
\centering
\includegraphics[width=\textwidth]{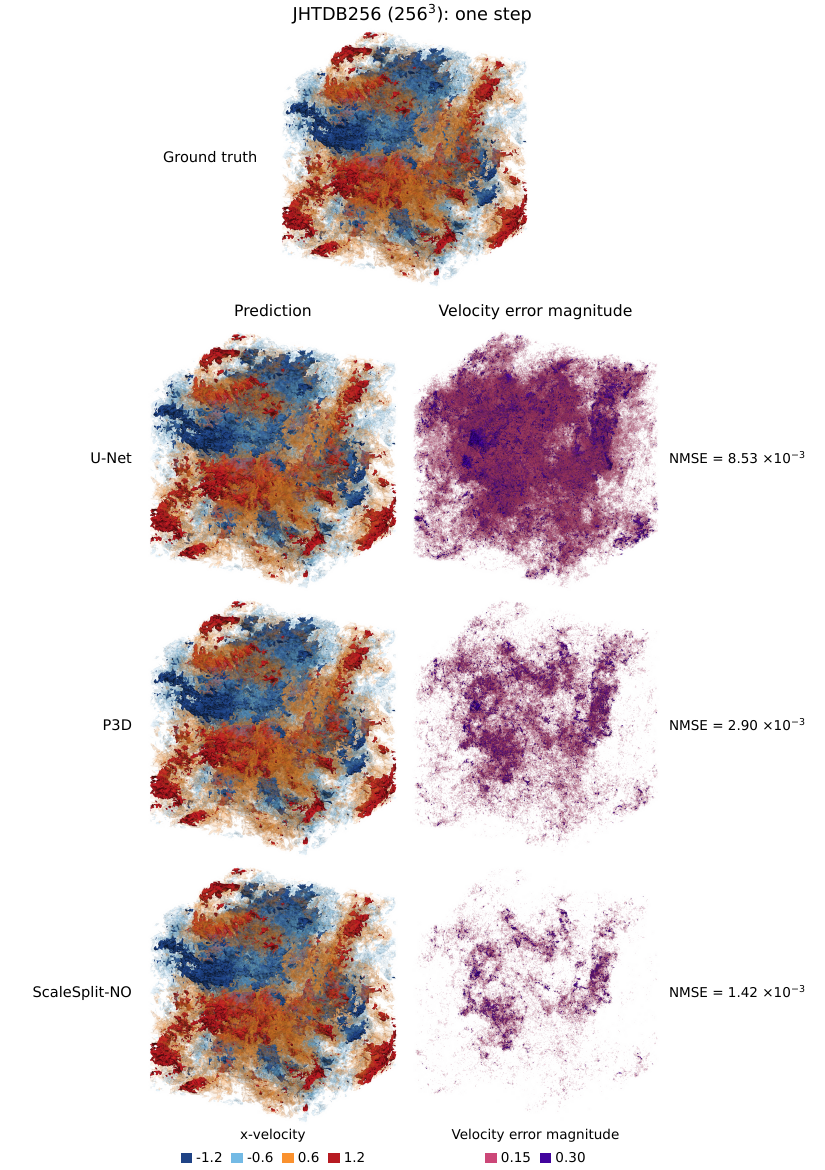}
\caption{\textbf{JHTDB256, complete field at one step.}
Isosurfaces of the x-velocity at one and two standard deviations of the ground truth, and of the velocity error magnitude at 0.15 and 0.30, for U-Net, P3D, and ScaleSplit-NO.}
\label{fig:app-jht-iso}
\end{figure}

\begin{figure}[p]
\centering
\includegraphics[width=\textwidth]{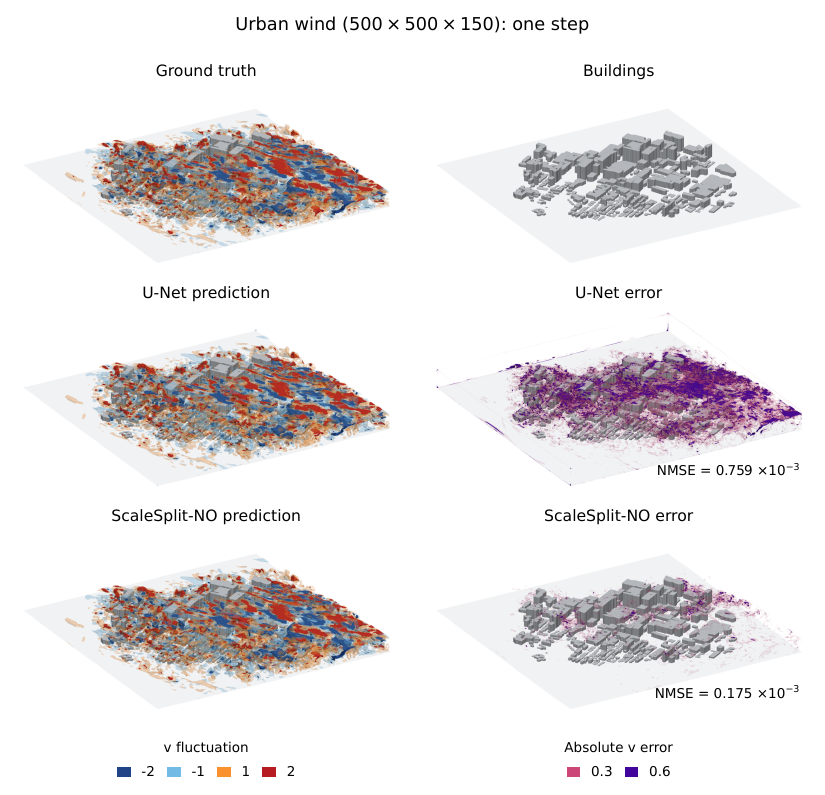}
\caption{\textbf{Urban wind at one step.}
Isosurfaces of the fluctuation of the velocity component $v$ about its temporal mean at $\pm1$ and $\pm2$, and of the absolute error of $v$ at 0.3 and 0.6, for U-Net and ScaleSplit-NO.
The NMSE refers to $v$ in the fluid region of this test pair.}
\label{fig:app-urban}
\end{figure}
\clearpage

\subsection{Energy spectra}
\label{sec:diagnostics}

Figures~\ref{fig:app-spectra-mhd} and~\ref{fig:app-spectra-jht} show the velocity energy spectra of the one-step predictions for the test pair of Figures~\ref{fig:app-mhd-methods}--\ref{fig:app-jht-iso}, together with the energy spectra of their errors.
On both datasets, the error spectrum of ScaleSplit-NO lies at or near the bottom at all wavenumbers, so that no baseline is clearly more accurate at any scale.
On MHD64, the error of ScaleSplit-NO is lower than that of every baseline at all wavenumbers.
On JHTDB256, it is lower than that of every baseline except P3D at all wavenumbers, and comparable to that of P3D at low to intermediate wavenumbers, so its overall advantage over P3D comes mainly from the smaller error at high wavenumbers.
On both datasets, its lead is larger at the lowest wavenumbers and at the high wavenumbers.

\begin{figure}[t]
\centering
\includegraphics[width=\textwidth]{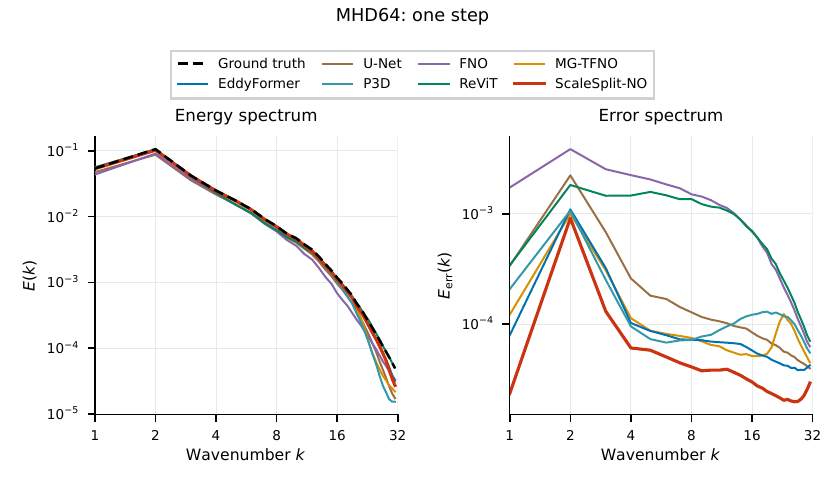}
\caption{\textbf{MHD64, velocity spectra at one step.}
Energy spectra of the ground truth and of all predictions (left) and of the prediction errors (right).}
\label{fig:app-spectra-mhd}
\end{figure}

\begin{figure}[t]
\centering
\includegraphics[width=\textwidth]{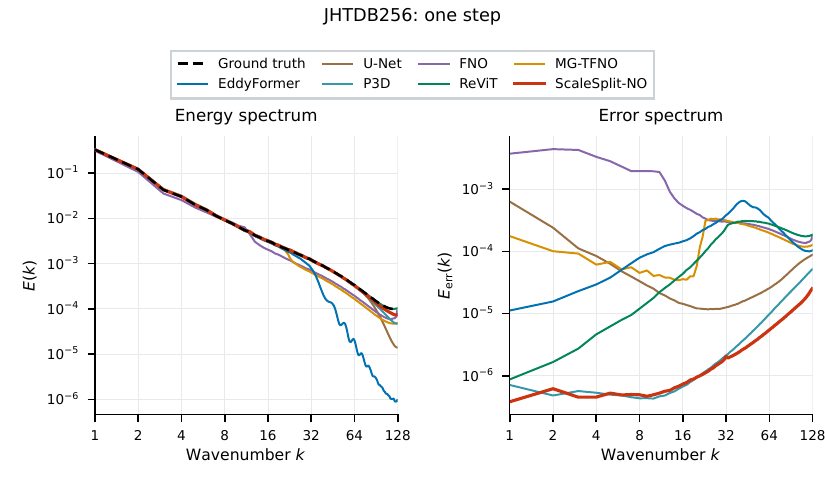}
\caption{\textbf{JHTDB256, velocity spectra at one step.}
Energy spectra of the ground truth and of all predictions (left) and of the prediction errors (right).}
\label{fig:app-spectra-jht}
\end{figure}

\subsection{Sensitivity to repeated training}
\label{app:sensitivity}
Due to limited computational resources, we could not repeat every experiment with full training.
To show how sensitive the results are to the random seed, each method in the low-data setting of both datasets and each ablation variant of Table~\ref{tab:ablation} is trained three times with different seeds at 20\% of its training cost, with its own learning-rate schedule compressed to this budget.
For ScaleSplit-NO the Parent is kept, and the Child is pretrained and fine-tuned again.
Tables~\ref{tab:sensitivity} and~\ref{tab:sensitivity-ablation} give the mean and the standard deviation of the one-step test NMSE.
Since these runs are short, the models are further from convergence than in the full training, and to different degrees.
Their errors are therefore generally higher.
ScaleSplit-NO still has the lowest mean error on both datasets.
Ablation variants with close errors in Table~\ref{tab:ablation} can change their order.
The overall trends agree with the full training, and for 20 of the 23 configurations the standard deviation is below 5\% of the mean, and at most 13\% in the remaining ones.

\begin{table}[t]
\caption{One-step test NMSE ($\times10^{-3}$) of all methods with 1/8 of the data over three repeated runs at 20\% of the training cost (seeds 101, 202, 303), mean $\pm$ sample standard deviation.}
\label{tab:sensitivity}
\begin{center}\small
\begin{tabular}{lcc}
\toprule
Method & MHD64 & JHTDB256\\
\midrule
U-Net & $21.45\pm0.34$ & $17.19\pm0.86$\\
FNO & $99.81\pm2.23$ & $99.50\pm10.61$\\
MG-TFNO & $9.67\pm0.34$ & $32.23\pm3.91$\\
EddyFormer & $12.45\pm0.23$ & $41.10\pm0.02$\\
P3D & $21.15\pm0.94$ & $4.07\pm0.04$\\
ReViT & $93.89\pm0.13$ & $32.96\pm0.09$\\
ScaleSplit-NO & $6.88\pm0.11$ & $1.53\pm0.01$\\
\bottomrule
\end{tabular}
\end{center}
\end{table}

\begin{table}[t]
\caption{One-step test NMSE ($\times10^{-3}$) of the ablation variants of Table~\ref{tab:ablation} on MHD64 with 1/8 of the data over three repeated runs at 20\% of the training cost, mean $\pm$ sample standard deviation.}
\label{tab:sensitivity-ablation}
\begin{center}\small
\begin{tabular}{lc}
\toprule
 & NMSE\\
\midrule
ScaleSplit-NO final configuration & $6.88\pm0.11$\\
\midrule
w/o pretraining and zero init. & $7.22\pm0.23$\\
w/o Child fine-tuning & $7.23\pm0.04$\\
Fine-tuning with coarse target & $12.04\pm0.48$\\
Residual correction & $7.33\pm0.14$\\
Joint Parent--Child training & $11.81\pm0.02$\\
\midrule
Global coarse grid $24^3\rightarrow16^3$ & $7.23\pm0.61$\\
Global coarse grid $24^3\rightarrow32^3$ & $7.09\pm0.15$\\
Local Child grid $16^3\rightarrow12^3$ & $6.76\pm0.11$\\
Local Child grid $16^3\rightarrow24^3$ & $9.08\pm0.05$\\
\bottomrule
\end{tabular}
\end{center}
\end{table}

\subsection{Swapping in more accurate Parents}
\label{app:transfer}
Table~\ref{tab:swap} gives the one-step test NMSE of every Child combined with every Parent trained on at least as much data, without retraining the Child.
The additional coarse fields are downsampled from the high-resolution simulation, so this experiment measures the efficiency of storage, not of generating the data.
Figure~\ref{fig:dataeff}(b1, b2) plots these values.
The last column replaces the coarse prediction by the coarse target $Du_{t+1}$, with the Child fixed.
The coarse target is not available at inference and serves as a reference.
On JHTDB256 the gain from a Parent trained on additional coarse fields is small, 1.9\% for 1.4\% of the full-data storage with the Child of 50 frames.
The added coarse data are few, and isotropic turbulence carries little global information, which P3D also exploits by omitting its context network for this flow.

\begin{table}[t]
\caption{One-step test NMSE ($\times10^{-3}$) of a fixed Child with Parents trained on more data, and with the coarse target as condition.
Rows: training data of the Child.
Columns: training data of the Parent.
Both are given as fractions of the full training set.}
\label{tab:swap}
\begin{center}\small
\begin{tabular}{llccccc}
\toprule
 & & \multicolumn{4}{c}{Parent} & \\
\cmidrule(lr){3-6}
Dataset & Child & 1/8 & 1/4 & 1/2 & 1 & Coarse target\\
\midrule
\multirow{4}{*}{MHD64} & 1/8 & 6.20 & 5.68 & 5.10 & 4.99 & 2.63\\
 & 1/4 & & 5.53 & 4.94 & 4.82 & 2.44\\
 & 1/2 & & & 4.79 & 4.77 & 2.48\\
 & 1 & & & & 4.52 & 2.35\\
\midrule
\multirow{4}{*}{JHTDB256} & 1/8 & 1.254 & 1.248 & 1.238 & 1.231 & 1.182\\
 & 1/4 & & 1.172 & 1.165 & 1.158 & 1.111\\
 & 1/2 & & & 1.078 & 1.071 & 1.025\\
 & 1 & & & & 1.025 & 0.980\\
\bottomrule
\end{tabular}
\end{center}
\end{table}

\subsection{Sizes of the Parent and the Child}
\label{app:ablations-results}
On MHD64 the Parent is global: it predicts the whole coarse field, so the Parent grid has the same size as the global coarse grid, and the two change together.
On JHTDB256 the Parent works on patches: it processes patches of the global coarse grid, so the Parent grid and the global coarse grid can be varied separately.
Table~\ref{tab:scale-jhtdb} gives the results on both datasets with 1/8 of the data, and every variant is trained with the complete training procedure.
On JHTDB256 the final configuration is among the most accurate ones, and only a smaller global coarse grid and a larger local Child grid are clearly worse.
A global coarse grid of $32^3$ is 10.5\% worse, and a finer one of $128^3$, which requires eight times more coarse data, brings no improvement.
Increasing the resolution of the global coarse grid has both an advantage and a disadvantage.
It adds detail to the Parent's prediction, which helps the Child only if this detail is accurate, but it also increases the number of values that a Parent model trained on very few coarse fields predicts, which makes overfitting more likely.
On MHD64 a finer grid of $32^3$ is likewise slightly worse than the final $24^3$.
On JHTDB256, smaller and larger Parent grids and a local Child grid of $12^3$ give similar errors, while a local Child grid of $24^3$ is 29.8\% worse.

\begin{table}[t]
\definecolor{worse}{HTML}{8C1616}
\caption{Sizes of the Parent and the Child with 1/8 of the data.
One-step test NMSE ($\times10^{-3}$), with the change relative to the final configuration of each dataset in parentheses, \mbox{\protect\textcolor{worse}{($\uparrow$ worse)}}.}
\label{tab:scale-jhtdb}
\begin{center}\small
\begin{tabular}{llr@{\,}l}
\toprule
Dataset & Setting & \multicolumn{2}{c}{NMSE}\\
\midrule
\multirow{5}{*}{MHD64} & \textbf{Final:} global coarse grid $=$ Parent grid $24^3$, local Child grid $16^3$ & 6.20 & \\
 & Global coarse grid $=$ Parent grid $24^3\rightarrow16^3$ & 6.20 & \textcolor{worse}{($\uparrow$0.1\%)}\\
 & Global coarse grid $=$ Parent grid $24^3\rightarrow32^3$ & 6.42 & \textcolor{worse}{($\uparrow$3.5\%)}\\
 & Local Child grid $16^3\rightarrow12^3$ & 6.30 & \textcolor{worse}{($\uparrow$1.6\%)}\\
 & Local Child grid $16^3\rightarrow24^3$ & 7.06 & \textcolor{worse}{($\uparrow$13.9\%)}\\
\midrule
\multirow{7}{*}{JHTDB256} & \textbf{Final:} global coarse grid $64^3$, Parent grid $32^3$, local Child grid $16^3$ & 1.25 & \\
 & Global coarse grid $64^3\rightarrow32^3$ & 1.39 & \textcolor{worse}{($\uparrow$10.5\%)}\\
 & Global coarse grid $64^3\rightarrow128^3$ & 1.26 & \textcolor{worse}{($\uparrow$0.35\%)}\\
 & Parent grid $32^3\rightarrow16^3$ & 1.30 & \textcolor{worse}{($\uparrow$3.96\%)}\\
 & Parent grid $32^3\rightarrow64^3$ & 1.32 & \textcolor{worse}{($\uparrow$4.96\%)}\\
 & Local Child grid $16^3\rightarrow12^3$ & 1.32 & \textcolor{worse}{($\uparrow$5.25\%)}\\
 & Local Child grid $16^3\rightarrow24^3$ & 1.63 & \textcolor{worse}{($\uparrow$29.8\%)}\\
\bottomrule
\end{tabular}
\end{center}
\end{table}

\subsection{Overlap of the Parent patches}
\label{app:overlap-results}
Table~\ref{tab:overlap-jhtdb} varies the overlap of the Parent patches on JHTDB256, with the Child overlap fixed at 27\%.
The overlap of the Parent patches changes the error by less than 3\% and the inference time by less than 2\%.

\begin{table}[t]
\caption{Overlap of the Parent patches on JHTDB256 with 1/8 of the data.
\textbf{Bold}: the default configuration.}
\label{tab:overlap-jhtdb}
\begin{center}\small
\begin{tabular}{rccc}
\toprule
Overlap & NMSE ($\times10^{-3}$) & Voxels & Time (s)\\
\midrule
0\% & 1.29 & 1.0$\times$ & 3.44\\
\textbf{33\%} & 1.25 & 3.4$\times$ & 3.46\\
50\% & 1.25 & 8.0$\times$ & 3.50\\
\bottomrule
\end{tabular}
\end{center}
\end{table}

\subsection{Urban wind fields}
\label{app:cityffd}
The CityFFD Montreal dataset \citep{mortezazadeh2022cityffd,qin2025modeling} contains the wind field of an urban district on a $500\times150\times500$ grid with a spacing of 4\,m, 1\,m, and 4\,m, with the three velocity components and the temperature as channels.
The buildings are given by the signed distance to their surfaces, which is appended to the four channels as a static input.
The prediction interval is 4\,s.
Frames 4300 to 4498 give 98 training pairs, frames 4500 to 4548 give 23 validation pairs, and frames 4550 to 4598 give 23 test pairs, and the NMSE is computed over the fluid region.

ScaleSplit-NO keeps the Parent and Child backbones of JHTDB256 and is adapted to the geometry as follows.
Patches are tiled without periodic wrap-around, with the last patch along each axis ending at the domain boundary, the coordinate channels give the absolute position in the domain, and the signed distance to the buildings is an additional input.
The Parent operates on $32^3$ patches of a $125\times50\times125$ coarse grid, and the Child on $16^3$ patches.
Both methods are trained on the same data and within 40 GB of GPU memory.
U-Net uses the architecture of Appendix~\ref{app:baselines} with the largest base width that fits the full field into this memory, and is trained for 300 epochs.
The test NMSE ($\times10^{-3}$) is 1.87 for ScaleSplit-NO and 5.46 for U-Net.
The peak GPU memory of one training update is 2.8 GiB for ScaleSplit-NO and 27.1 GiB for U-Net.
Figure~\ref{fig:app-urban} shows the first test pair.

%% file: sections/appendix/datasets.tex
\section{Datasets and evaluation protocol}
\label{app:protocol}
\subsection{Datasets and splits}
MHD64 is the $64^3$ magnetohydrodynamic subset of The Well (\texttt{MHD\_64}), which simulates compressible turbulence in the magnetized interstellar medium, at sonic and Alfv\'enic Mach numbers of 0.7 \citep{the_well,burkhart2020catalogue}.
It contains ten trajectories of 100 frames at a time interval of 0.01, with density, three velocity components, and three magnetic-field components, which are coupled, so a surrogate learns several interacting fields at once.
The trajectories start from different initial conditions.
We keep the released split of eight training trajectories, one for validation, and one for testing, and the four data budgets use one, two, four, and all eight training trajectories.

JHTDB256 is derived from the forced isotropic turbulence of the Johns Hopkins Turbulence Database (\texttt{isotropic1024coarse}), a direct numerical simulation at a Taylor-scale Reynolds number of 433 that provides velocity and pressure on a periodic $1024^3$ grid \citep{perlman2007data,li2008public}.
At its original resolution of $1024^3$, the simulation resolves all scales of the flow, with a clear separation between the energy-containing and the dissipative scales.
We take every fourth grid point along each axis, giving $256^3$ fields that resolve wavenumbers up to 128, and every tenth archived snapshot, giving a time interval of 0.02 between consecutive frames.
At this resolution, all baselines use relatively small configurations to be trained on a single 40 GB GPU (Appendix~\ref{app:baselines}).
The four training budgets use the archived steps 3501 to 3991, 3001 to 3991, 2001 to 3991, and 1 to 3991 in strides of ten, that is, 50, 100, 200, and 400 frames, so that the smaller budgets keep the frames closest to the test window.
Validation uses the steps 4001 to 4491 and testing the steps 4501 to 4991.

MHD64 hence tests generalization to an unseen trajectory, and JHTDB256 tests prediction at later times of one flow.
In the low-data setting with 1/8 of the data, MHD64 loses seven of its eight initial conditions and not only frames, so it is the more demanding of the two low-data tests.

\subsection{Evaluation protocol}
Single-step evaluation uses all 99 test pairs of MHD64 and all 49 test pairs of JHTDB256.
For autoregressive rollouts, the MHD64 test trajectory is split into ten segments of ten frames and the JHTDB256 test window into three segments of sixteen frames.
Each segment starts from the simulated field of its first frame and is predicted for 9 and 15 steps, and every method receives its own previous prediction as input.
Table~\ref{tab:main} reports the mean over the first two steps on MHD64 and over all 15 steps on JHTDB256.
These horizons correspond to a similar loss of correlation, since the flow of MHD64 changes much faster per time step than that of JHTDB256.
With every channel shifted and scaled by the training statistics and all channels flattened, the cosine similarity between a simulated field and the simulated field two steps earlier is 0.81 on MHD64, and the same value is reached after 15 steps on JHTDB256.
Longer rollouts on MHD64 are of limited use, since on the complete magnetohydrodynamic dataset of The Well the errors of the benchmark models exceed that of simply using the mean field as the prediction within a dozen steps \citep{the_well}.

\subsection{Metrics}
\label{app:metrics}
The NMSE defined in Section~\ref{sec:protocol} is the main metric.
Appendix~\ref{app:results} also reports the normalized root mean squared error, $\mathrm{NRMSE}=\frac1q\sum_k\|F(u_t)_k-u_{t+1,k}\|_2/\|u_{t+1,k}\|_2$.
Both are computed on the complete field in physical units and averaged over the test pairs, and for rollouts over the segments at each step.
All percentages in the paper, including the appendix, are computed from the unrounded results and not from the rounded values shown in the tables.

%% file: sections/appendix/implementation.tex
\section{Implementation and training details}
\label{app:implementation}
\subsection{Architecture}
\paragraph{Grids and patches.}
MHD64 uses a $64^3$ fine grid, a $24^3$ coarse grid, and $16^3$ Child patches, and its Parent processes the whole coarse field.
JHTDB256 uses a $256^3$ fine grid, a $64^3$ coarse grid, $32^3$ Parent patches, and $16^3$ Child patches.
The restriction $D$ keeps the frequencies representable on the coarse grid, and the interpolation $I$ pads the remaining frequencies with zeros.

\paragraph{FNO backbones.}
Both models are FNOs with residual Fourier blocks \citep{FNO}, each with instance normalization and a pointwise MLP, and three coordinate channels are appended to the input.
Table~\ref{tab:architecture-parameters} lists the configurations.
The Parent has half the depth and half the width of the Child.
On JHTDB256 the two models still have a similar number of parameters, since the Parent keeps twice as many Fourier modes per axis on its $32^3$ patch as the Child on its $16^3$ patch, and the number of spectral weights grows with the cube of this number.
The Child architecture is the same on both datasets up to the number of physical channels.

\begin{table}[t]
\caption{Backbone configurations.
Modes are the number of Fourier modes per axis.
The parameters of the Child include the condition pathway.}
\label{tab:architecture-parameters}
\begin{center}\small
\begin{tabular}{lcccc}
\toprule
 & \multicolumn{2}{c}{Parent} & \multicolumn{2}{c}{Child}\\
\cmidrule(lr){2-3}\cmidrule(l){4-5}
 & MHD64 & JHTDB256 & MHD64 & JHTDB256\\
\midrule
Input & $24^3$ field & $32^3$ patch & $16^3$ patch & $16^3$ patch\\
Fourier blocks & 4 & 4 & 8 & 8\\
Hidden width & 32 & 32 & 64 & 64\\
Modes per axis & 12 & 16 & 8 & 8\\
Parameters ($10^6$) & 56.6 & 134.2 & 134.3 & 134.3\\
\bottomrule
\end{tabular}
\end{center}
\end{table}

\paragraph{Parent prediction.}
\label{app:physical-maps}
The Parent predicts the increment from the coarse input to the coarse target, and the increment is added to the coarse input to give the coarse prediction.
On JHTDB256 the increments of the overlapping Parent patches are assembled on the coarse grid before they are added.

\paragraph{Condition pathway.}
\label{app:training-maps}
The normalized input patch and the normalized condition are concatenated before the lifting layer.
In addition, nine pointwise linear maps take the condition to the hidden width, and their outputs are added after the lift and after each of the eight Fourier blocks.
At the start of fine-tuning, the condition columns of the lift and all nine maps, including their biases, are zero, and the remaining weights are those of the pretrained Child.

\begin{lemma}[Function-preserving initialization]
\label{lem:zero}
With this initialization, $C_{\theta,0}(x,c)=C^{\mathrm{pre}}_{\theta}(x)$ for every input patch $x$ and every condition $c$.
\end{lemma}
\paragraph{Proof.}
The condition enters the Child only through the condition columns of the lift and through the added outputs of the nine maps, which are all zero.
The hidden state after the lift therefore equals that of the pretrained Child, and since every block and the projection are unchanged, so do all later hidden states and the output.

\paragraph{Patch layout and assembly.}
On MHD64 and JHTDB256, the patches are placed at inference at nearly uniform spacing with periodic wrap-around.
This gives 6 Child patches per axis on MHD64 with an average overlap of 33\%, 22 on JHTDB256 with an average overlap of 27\%, and 3 Parent patches per axis on the coarse grid of JHTDB256 with an average overlap of 33\%.
Overlapping predictions are weighted by a Hann window without its zero end points and divided by the sum of the weights, as in Equation~\eqref{eq:assembly}.

\paragraph{Patch sampling during training.}
In every epoch, training patches are drawn at uniformly random positions, and the input patch, the target patch, and the condition share the same position.
On MHD64, 64 Child patches are drawn from every frame pair per epoch.
On JHTDB256, 256 Child patches and 8 Parent patches are drawn from every frame pair per epoch.

\paragraph{Normalization.}
Every channel is shifted and scaled with the mean and standard deviation of the training frames of the respective data budget.
A Parent keeps its own statistics.

\subsection{Training procedure}\label{app:training}
The three stages are run one after another on one GPU with AdamW, weight decay $10^{-4}$, and a cosine learning-rate schedule that decays to zero.
Table~\ref{tab:training-settings} lists the epochs, learning rates, and batch sizes.
The Parent is trained with the 48 rotations and reflections of the cube as augmentation, and the Child is trained without augmentation in both pretraining and fine-tuning.
Each Child is fine-tuned with the Parent trained on the same data.
Checkpoints are selected by the mean squared error in physical units on the validation split, and the selected checkpoints are evaluated once on the test split.

\paragraph{Training objectives.}
All three stages minimize the normalized root mean squared error (NRMSE), $\ell(\widehat v,v)=\frac1q\sum_{k=1}^{q}\frac{\|\widehat v_k-v_k\|_2}{\|v_k\|_2}$, where $\widehat v$ is a prediction, $v$ its target, $k$ indexes the $q$ channels, and the norms run over the grid points of the target.
The loss is computed in normalized units, on the outputs of the networks before the normalization is undone.
For patch $j$, the normalized input patch, condition, and target patch are
\[
\bar x_j=\mathsf N_xT_ju_t,\qquad
\bar c_j=\mathsf N_cc_j,\qquad
\bar y_j=\mathsf N_yT_ju_{t+1},
\]
with fixed normalization maps $\mathsf N_x$, $\mathsf N_c$, and $\mathsf N_y$.
The Child map of Equation~\eqref{eq:composite-map} in physical units is
\[
C_{\theta,\psi}(x,c)=\mathsf N_y^{-1}\big[\widetilde C_{\theta,\psi}(\mathsf N_xx,\mathsf N_cc)\big],
\]
where $\widetilde C_{\theta,\psi}$ is the Child network on normalized fields, and $\widetilde C^{\mathrm{pre}}_\theta$ denotes the same network without the condition pathway, with $C^{\mathrm{pre}}_\theta(x)=\mathsf N_y^{-1}\big[\widetilde C^{\mathrm{pre}}_\theta(\mathsf N_xx)\big]$.
The three objectives are
\begin{equation}
\begin{aligned}
 \text{Parent:}\quad&\min_\phi\ \mathbb E\,\ell\big(G_\phi(\mathsf N_zDu_t),\ \mathsf N_\Delta(Du_{t+1}-Du_t)\big),\\
 \text{Child pretraining:}\quad&\min_\theta\ \mathbb E\,\ell\big(\widetilde C^{\mathrm{pre}}_\theta(\bar x_j),\ \bar y_j\big),\\
 \text{Child fine-tuning:}\quad&\min_{\theta,\psi}\ \mathbb E\,\ell\big(\widetilde C_{\theta,\psi}(\bar x_j,\bar c_j),\ \bar y_j\big)\quad\text{with }\phi\text{ fixed},
\end{aligned}
\label{eq:train-cond-main}
\end{equation}
where $G_\phi$ is the Parent network of Equation~\eqref{eq:parent-decoding-main} and the expectation runs over the training frame pairs and patch positions.
On JHTDB256, $Du_t$ and $Du_{t+1}$ in the Parent objective are patches of the coarse grid.

\begin{table}[t]
\caption{Training settings of the three stages.}
\label{tab:training-settings}
\begin{center}\small
\begin{tabular}{llccc}
\toprule
Dataset & Stage & Epochs & Learning rate & Batch\\
\midrule
\multirow{3}{*}{MHD64} & Parent training & 300 & $2\times10^{-3}$ & 1 field\\
 & Child pretraining & 50 & $1\times10^{-3}$ & 8 patches\\
 & Child fine-tuning & 300 & $2\times10^{-3}$ & 8 patches\\
\midrule
\multirow{3}{*}{JHTDB256} & Parent training & 300 & $2\times10^{-3}$ & 8 patches\\
 & Child pretraining & 100 & $5\times10^{-4}$ & 16 patches\\
 & Child fine-tuning & 500 & $2\times10^{-3}$ & 16 patches\\
\bottomrule
\end{tabular}
\end{center}
\end{table}

\paragraph{Ablation variants.}\label{app:ablation}
The variants of Table~\ref{tab:ablation} are trained on 1/8 of the MHD64 data with the settings above.
Without Child fine-tuning, a Child is trained for 350 epochs without condition.
Without pretraining and zero initialization, a Child with randomly initialized backbone and condition pathway is trained for 350 epochs with the condition.
Fine-tuning with coarse target uses the pretrained Child and the coarse target as condition, and is evaluated with the coarse prediction.
The residual variant is trained for 350 epochs to predict the difference between the target patch and the interpolated coarse prediction.
Joint training starts from the pretrained Parent and Child, updates both through the Child loss for 300 epochs, and uses the updated Parent at inference.
The size variants retrain the affected model with the settings above and fine-tune the Child again.

\subsection{Baseline configurations}
\label{app:baselines}
Table~\ref{tab:baseline-settings} lists the configurations of the six baselines.
All baselines are trained on the same splits as ScaleSplit-NO without augmentation, and select their checkpoints by the mean squared error in physical units on the validation split.
Each baseline is trained with its own loss, U-Net, FNO, and EddyFormer with the mean squared error on normalized fields, ReViT and P3D with the mean squared error on physical fields, and MG-TFNO with the $H^1$ loss on normalized fields.
On JHTDB256, every method is trained within the same 40 GB of GPU memory, and all baselines use relatively small configurations (Table~\ref{tab:baseline-settings}).
U-Net, FNO, and MG-TFNO use a smaller width, FNO and MG-TFNO fewer Fourier modes, and ReViT a patch size of 4 in place of 2, with activation checkpointing for MG-TFNO and ReViT.
P3D uses its smaller P3D-B backbone on $128^3$ crops without a context network, and EddyFormer keeps the number of tokens of its $96^3$ setting.
MG-TFNO is proposed for two-dimensional fields, and we extend its multi-grid decomposition to three dimensions.
MG-TFNO processes the patches of a field in parallel but computes one loss on the prediction assembled from all of them, so every update still requires the whole field.
Its patches thus allow the memory to be distributed over several GPUs, which lowers the memory per GPU but not the total memory.
The complete P3D method fine-tunes a global context network on the whole domain, since crops alone miss global information \citep{P3d}.
For isotropic turbulence, which has weak global features, the P3D authors can omit the global context network, but the crops still need to be as large as $128^3$, since smaller crops give higher errors.
We follow this configuration on JHTDB256, while on MHD64 P3D is trained on the whole field.

\begin{table}[t]
\caption{Baseline configurations on MHD64 / JHTDB256.
Width is the base channel count of U-Net, the hidden width of FNO, MG-TFNO, and EddyFormer, and the embedding dimension of ReViT.
EddyFormer is trained for 10k steps with 1/8 of the data and 80k steps with all of the data.}
\label{tab:baseline-settings}
\begin{center}\small\setlength{\tabcolsep}{4pt}
\begin{tabular}{lllllll}
\toprule
Method & Width & Depth & Modes & Epochs & Learning rate & Batch\\
\midrule
U-Net & 64 / 16 & 4 levels & -- & 300 & $2\times10^{-4}$ & 2 / 1\\
FNO & 48 / 12 & 4 layers & 16 / 12 & 400 & $1\times10^{-3}$ & 1\\
MG-TFNO & 64 / 40 & 4 layers & 32 / 24 & 500 / 200 & $1\times10^{-3}$ & 1 field\\
EddyFormer & 32 & 4 layers & 13 & 10k--80k steps & $1\times10^{-3}$ & 1\\
P3D & P3D-L / P3D-B & & -- & 1,000 / 4,000 & $2\times10^{-4}$ & 8 / 4\\
ReViT & 48 & 1-2-4-2-1 & -- & 300 & $1\times10^{-3}$ & 2 / 1\\
\bottomrule
\end{tabular}
\end{center}
\end{table}

\subsection{Training time, inference time, and memory}
\label{app:resources}
All times were measured on the same A100 with 40 GB, with no other job running on it.
The training memory is measured on an RTX PRO 6000 with 96 GB, so that the memory of the baselines can also be measured at batch sizes that exceed 40 GB (Figure~\ref{fig:batch-memory}).
The same measurements on the A100 are nearly identical, with differences of at most a few MB.
Table~\ref{tab:train-config} lists the training memory and time, and Table~\ref{tab:inference} the inference time.
Child fine-tuning takes about 80\% of the training time of ScaleSplit-NO.
Without overlap, ScaleSplit-NO is about as fast as the baselines on MHD64.
On JHTDB256 it is slower, because all baselines use relatively small configurations there (Appendix~\ref{app:baselines}).
With overlap, the Child predicts more voxels, so inference takes longer, roughly in proportion to the number of voxels (Table~\ref{tab:inference}).
The training memory is the peak GPU memory allocated during one training update, including the model, gradients, and optimizer states, and excluding validation and the caching of data.

Figure~\ref{fig:batch-memory} shows how the training memory changes with the batch size.
The memory is measured in the same way as in Table~\ref{tab:train-config}.
For every baseline, the memory grows roughly in proportion to the batch size, since each sample is a whole field or a large crop.
Each sample of ScaleSplit-NO is a patch or a coarse field, and on JHTDB256 all three of its stages need at most 4.6 GiB up to a batch of 16.
A single sample of any baseline on JHTDB256 already needs at least 7.0 GiB.
For MG-TFNO the batch counts the patches processed together, with one whole field per update, and all points use activation checkpointing.
On JHTDB256, EddyFormer at a batch of 4 and MG-TFNO at a batch of 8 do not fit into 96 GB and are omitted.

\begin{table}[t]
\caption{\textbf{Training samples.}
Input is the grid size of one training sample, batch the number of samples per update, memory the peak GPU memory allocated during one training update, and time the training time with all of the data.
For ScaleSplit-NO the first row gives the maximum memory and the total time of its three stages, which are trained one after another and listed below it.
MG-TFNO splits each field into eight patches and computes one loss on the prediction assembled from all of them, so every update uses one whole field.
For MG-TFNO, Input is the size of one patch without padding, $48^3$ and $144^3$ with periodic padding, and Batch the number of patches processed together.}
\label{tab:train-config}
\begin{center}\small
\begin{tabular}{lcccccccc}
\toprule
& \multicolumn{4}{c}{MHD64} & \multicolumn{4}{c}{JHTDB256}\\
\cmidrule(lr){2-5}\cmidrule(l){6-9}
 & Input & Batch & Memory & Time & Input & Batch & Memory & Time\\
Method & & & GiB & h & & & GiB & h\\
\midrule
U-Net & $64^3$ & 2 & 4.6 & 4.6 & $256^3$ & 1 & 31.7 & 41.3\\
FNO & $64^3$ & 1 & 5.5 & 4.6 & $256^3$ & 1 & 35.0 & 41.4\\
MG-TFNO & $32^3$ & 8 & 10.6 & 20.1 & $128^3$ & 1 & 19.8 & 32.7\\
EddyFormer & $64^3$ & 1 & 17.0 & 16.7 & $256^3$ & 1 & 26.4 & 22.9\\
P3D & $64^3$ & 8 & 16.3 & 14.1 & $128^3$ & 4 & 25.6 & 28.3\\
ReViT & $64^3$ & 2 & 2.3 & 2.6 & $256^3$ & 1 & 14.3 & 70.6\\
\midrule
ScaleSplit-NO & -- & -- & 2.2 & 9.7 & -- & -- & 3.0 & 25.5\\
\quad Stage 1: Parent training & $24^3$ & 1 & 0.9 & 0.7 & $32^3$ & 8 & 3.0 & 1.0\\
\quad Stage 2: Child pretraining & $16^3$ & 8 & 2.2 & 1.2 & $16^3$ & 16 & 2.8 & 3.7\\
\quad Stage 3: Child fine-tuning & $16^3$ & 8 & 2.2 & 7.7 & $16^3$ & 16 & 2.8 & 20.8\\
\bottomrule
\end{tabular}
\end{center}
\end{table}

\begin{table}[t]
\caption{\textbf{Inference time.}
Time for one complete field.
For ScaleSplit-NO the time is given for several overlaps of the Child patches.}
\label{tab:inference}
\begin{center}\small
\begin{tabular}{lcccc}
\toprule
& \multicolumn{2}{c}{MHD64} & \multicolumn{2}{c}{JHTDB256}\\
\cmidrule(lr){2-3}\cmidrule(l){4-5}
Method & Overlap & Time (ms) & Overlap & Time (s)\\
\midrule
U-Net & -- & 28 & -- & 0.31\\
FNO & -- & 15 & -- & 0.41\\
MG-TFNO & -- & 55 & -- & 0.71\\
EddyFormer & -- & 341 & -- & 0.38\\
P3D & -- & 35 & -- & 0.41\\
ReViT & -- & 23 & -- & 0.44\\
\midrule
\multirow{5}{*}{ScaleSplit-NO} & 0\% & 26 & 0\% & 1.35\\
 & 20\% & 47 & 16\% & 2.24\\
 & 33\% & 81 & 27\% & 3.46\\
 & 43\% & 122 & 38\% & 5.68\\
 & 50\% & 182 & 50\% & 10.56\\
\bottomrule
\end{tabular}
\end{center}
\end{table}

\begin{figure}[t]
\centering
\includegraphics[width=\linewidth]{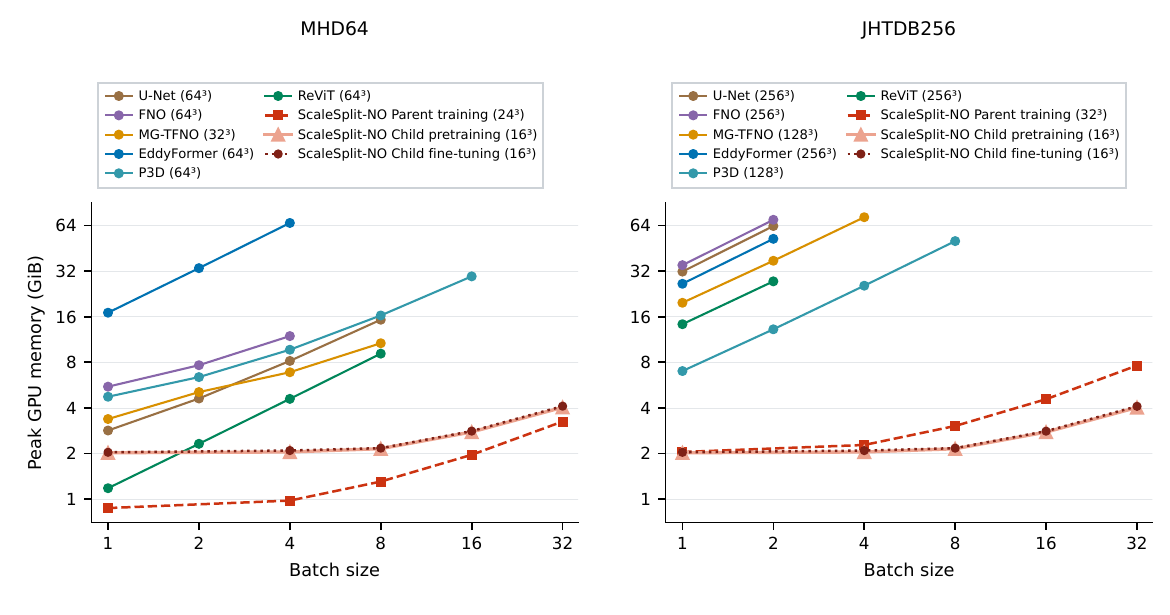}
\caption{\textbf{Training memory versus batch size.}
Peak GPU memory of one training update, with the input size of one sample in the legend.
ScaleSplit-NO is shown for each of its three training stages.}
\label{fig:batch-memory}
\end{figure}

%% file: sections/appendix/theory.tex
\section{Supplementary theory}
\label{app:theory}

This appendix analyzes the prediction interface of ScaleSplit-NO for fixed predictors and a specified test distribution.

\subsection{Conventions}
\label{app:theory-conventions}

Let $\mathcal H_N=\mathbb R^{q\times n}$ and
$\mathcal H_R=\mathbb R^{q\times r}$, where $n=N^3$ and $r=R^3$.
We use the physical-field norms
\begin{equation}
 \|v\|_N^2=\frac{1}{qn}\sum_x\|v(x)\|_2^2,
 \qquad
 \|z\|_R^2=\frac{1}{qr}\sum_\xi\|z(\xi)\|_2^2,
 \qquad
 \langle v,w\rangle_N=\frac{1}{qn}\sum_x v(x)^\top w(x).
 \label{eq:theory-norms}
\end{equation}
Write $U=u_t$ and $V=u_{t+1}$ for a test pair with law $\mu$ and finite second moments.
A deterministic evolution law $V=\mathcal S_{\Delta t}(U)$ is
allowed but is not required for the risk identities.
All trained parameters,
normalization maps, patch locations, and assembly weights are fixed, as is the
test distribution $\mu$.
The physical-field MSE risk
$\mathcal R_\mu(F)=\mathbb E\|F(U)-V\|_N^2$ differs from both the NRMSE training objective and the target-normalized NMSE reported in Section~\ref{sec:protocol}.
The bounds below are stated for this MSE risk.

All prediction and preprocessing maps are measurable.
The restriction $D:\mathcal H_N\to\mathcal H_R$ and interpolation
$I:\mathcal H_R\to\mathcal H_N$ are the fixed linear maps of Section~\ref{sec:formulation}.
The restriction $T_j$ extracts
the patch supported on $\Omega_j$.
The physical Parent map $P_\phi$ includes
any coarse-patch assembly and conversion from normalized network outputs to
physical states.
Write $\vartheta=(\theta,\psi)$, so that $C_\vartheta=C_{\theta,\psi}$ is the physical Child map of Equation~\eqref{eq:composite-map}, including the fixed normalization and decoding specified in Appendix~\ref{app:training-maps}.
The assembly $A$ is the map of Equation~\eqref{eq:assembly}.
For each patch, let $X_j=T_jU$, $Y_j=T_jV$, and $c_j(u)=T_jIP_\phi(Du)$; its random forecast condition is $C_j=c_j(U)$.
Thus
\begin{equation}
 F_{\phi,\vartheta}(U)
 =A\big((C_\vartheta(T_jU,T_jIP_\phi(DU)))_j\big).
 \label{eq:theory-composed-map}
\end{equation}
As in the main text, $F=F_{\phi,\vartheta}$ when both predictors are fixed.
With the Child fixed, we write $F_P$ to emphasize the physical Parent map $P$ used in this composition; $P=P_\phi$ for the trained Parent.

\subsection{Conditional information and coarse prediction}
\label{app:theory-information}

Fix a patch index $j$, and write
$X=T_jU$, $Y=T_jV$, $Z=DU$, and $C=T_jIP(Z)$.
Only within a fixed-patch argument, we suppress the index on $X_j,Y_j,C_j$; the risk retains its subscript $j$ to identify the target.
Here $C$ denotes a random condition, whereas $C_\vartheta$ denotes the Child map.
For an observation $W$, define the unrestricted patch risk:
\begin{equation}
 \mathcal R_j^*(W)=\inf_{f\ \mathrm{measurable}}\mathbb E\|Y-f(W)\|_2^2.
 \label{eq:theory-bayes-risk}
\end{equation}
The same identities hold after any fixed positive normalization of this norm.

\begin{lemma}[Value and limitation of forecast conditions]
\label{prop:information}
Under the fixed-patch setup above and $\mathbb E\|Y\|_2^2<\infty$,
\begin{align*}
 \mathcal R_j^*(X)-\mathcal R_j^*(X,C)
 &=\mathbb E\left\|\mathbb E[Y\mid X,C]-\mathbb E[Y\mid X]\right\|_2^2,\\
 \mathcal R_j^*(X,Z)&\le \mathcal R_j^*(X,C)\le \mathcal R_j^*(X).
\end{align*}
\end{lemma}

\paragraph{Proof.}
Let $m_W=\mathbb E[Y\mid W]$.
For any square-integrable $W$-measurable
prediction $f(W)$, conditional expectation gives
\begin{equation}
 \mathbb E\|Y-f(W)\|_2^2
 =\mathbb E\|Y-m_W\|_2^2
  +\mathbb E\|m_W-f(W)\|_2^2.
 \label{eq:theory-projection}
\end{equation}
Indeed, the cross term is zero because
$\mathbb E[Y-m_W\mid W]=0$.
Predictions of infinite risk cannot improve the
infimum, so $\mathcal R_j^*(W)=\mathbb E\|Y-m_W\|_2^2$.

Now let $m_X=\mathbb E[Y\mid X]$ and
$m_{XC}=\mathbb E[Y\mid X,C]$.
Decomposing
$Y-m_X=(Y-m_{XC})+(m_{XC}-m_X)$ yields
\begin{equation}
 \mathcal R_j^*(X)-\mathcal R_j^*(X,C)
 =\mathbb E\|m_{XC}-m_X\|_2^2\ge0.
 \label{eq:theory-information-gain}
\end{equation}
Since $P$ is measurable and $I$ and $T_j$ are linear, $C$ is a measurable function of $Z$.
Consequently,
$\sigma(X)\subseteq\sigma(X,C)\subseteq\sigma(X,Z)$, and the same argument
gives
\begin{align}
 \mathcal R_j^*(X,C)-\mathcal R_j^*(X,Z)
 &=\mathbb E\left\|
     \mathbb E[Y\mid X,Z]-\mathbb E[Y\mid X,C]
   \right\|_2^2\ge0,\label{eq:theory-information-loss}\\
 \mathcal R_j^*(X,Z)&\le \mathcal R_j^*(X,C)\le \mathcal R_j^*(X).
 \label{eq:theory-information-order}
\end{align}
If the patch index is random, its observed value $J$ must be included in all
conditioning variables for this argument.

Strict improvement over $X$ occurs exactly when the two conditional means in
\eqref{eq:theory-information-gain} differ with positive probability.
This is a
condition on the conditional mean, not an equivalence with full conditional
independence.
The variable $Z$ in \eqref{eq:theory-information-order} is the
complete current coarse field; replacing it by a restricted coarse patch changes
the information comparison.
A forecast can expose relevant nonlocal information
to a patch-only predictor, but cannot create information beyond $(X,Z)$.

\paragraph{When a cropped current coarse field retains the complete observation.}
Assume a Cartesian Fourier band with $b_\ell\le N$ consecutive integer frequencies modulo $N$ in axis $\ell$, and let $I$ be a linear bijection from coarse coordinates onto this band space.
If a rectangular patch contains $p_\ell\ge b_\ell$ consecutive fine-grid samples in every axis, then $T_jI$ is injective in exact arithmetic.
Consequently, for $C_j^0=T_jIZ$,
\begin{equation}
 \mathcal R_j^*(X_j,C_j^0)=\mathcal R_j^*(X_j,Z)\le \mathcal R_j^*(X_j,C_j).
 \label{eq:band-crop-information}
\end{equation}
\emph{Proof.}
Take the first $b_\ell$ patch samples in each axis.
The one-dimensional evaluation matrix has entries
$\exp(2\pi\mathrm{i}(k_{0,\ell}+k)(x_{0,\ell}+t)/N)$ for $t,k=0,\ldots,b_\ell-1$.
Factoring the row and column phases leaves a Vandermonde matrix on the distinct nodes $\exp(2\pi\mathrm{i}t/N)$, so this matrix is invertible.
The tensor product is invertible, giving a linear inverse on the range of $T_jI$.
The result also holds on the conjugate-symmetric subspace representing real fields.
Thus $Z$ is measurable from $C_j^0$, proving the equality; the inequality follows from Lemma~\ref{prop:information}.

Here $b_\ell$ counts modes in the interpolation output, not a coarsening stride or the modes retained inside an FNO layer.

\paragraph{Coarse-prediction decomposition.}
Let $Z^+=DV$, $m_{\mathrm c}(Z)=\mathbb E[Z^+\mid Z]$, and
$\sigma_{\mathrm{cl}}^2:=\mathbb E\|Z^+-m_{\mathrm c}(Z)\|_R^2$.
Then $Z^+$ and $m_{\mathrm c}(Z)$ are square-integrable and $\sigma_{\mathrm{cl}}<\infty$, since $D$ is linear and $V$ has a finite second moment.
For any square-integrable Parent output, the same orthogonality argument in the normalized coarse norm yields
\begin{equation}
 \mathbb E\|P(Z)-Z^+\|_R^2
 =\mathbb E\|Z^+-\mathbb E[Z^+\mid Z]\|_R^2
 +\mathbb E\|P(Z)-\mathbb E[Z^+\mid Z]\|_R^2.
 \label{eq:theory-coarse-decomposition}
\end{equation}
The first term need not vanish: the chosen coarse observation need not determine
the coarse future.

\paragraph{What the coarse residual includes.}
The formulation permits a future that is not determined by the observed $U$.
For square-integrable $Z^+=DV$, let $m_{\mathrm c}^U=\mathbb E[Z^+\mid U]$ and use $m_{\mathrm c}(Z)$ defined above.
Since $DU$ is a function of $U$, conditional projection gives
\begin{equation}
 \sigma_{\mathrm{cl}}^2
 =\underbrace{\mathbb E\|Z^+-m_{\mathrm c}^U\|_R^2}_{\text{uncertainty given the full observation}}
 +\underbrace{\mathbb E\|m_{\mathrm c}^U-m_{\mathrm c}(Z)\|_R^2}_{\text{additional loss from coarse observation}}.
 \label{eq:coarse-uncertainty-split}
\end{equation}
Indeed, $Z^+-m_{\mathrm c}(Z)=(Z^+-m_{\mathrm c}^U)+(m_{\mathrm c}^U-m_{\mathrm c}(Z))$; the cross term vanishes
because $m_{\mathrm c}^U-m_{\mathrm c}(Z)$ is $U$-measurable and $\mathbb E[Z^+-m_{\mathrm c}^U\mid U]=0$.
When $V=\mathcal S_{\Delta t}(U)$, the first term is zero, and the
coarse residual is entirely due to replacing $U$ by $DU$.
Otherwise the first term can be positive even when coarse graining loses no
information relevant to the conditional mean.
For example, if $U=0$ and
$V=H\mathbf1$ for a symmetric random sign $H$, mean restriction gives
first term one and second term zero.
The Parent-error identity remains valid in both cases.

\subsection{Weighted assembly}
\label{app:theory-assembly}

Assume a finite patch family with fixed weights satisfying
$a_j(x)\ge0$, $a_j(x)=0$ outside $\Omega_j$, and
$\sum_j a_j(x)=1$ at every voxel.
Define
\begin{equation}
 (A\mathbf v)(x)=\sum_j a_j(x)v_j(x),
 \qquad
 \|\mathbf v\|_a^2
 =\frac{1}{qn}\sum_x\sum_j a_j(x)\|v_j(x)\|_2^2,
 \qquad Tv=(T_jv)_j.
 \label{eq:theory-assembly-def}
\end{equation}
Only values on patch supports enter these expressions.
The quantity
$\|\cdot\|_a$ is a seminorm if some patch coordinates have zero weight.

\begin{lemma}[Weighted assembly]
\label{lem:assembly}
Under the fixed nonnegative partition-of-unity weights above, for every full field $v$ and patch collection $\mathbf v$,
\begin{equation}
 A Tv=v,
 \qquad \|Tv\|_a=\|v\|_N,
 \qquad \|A\mathbf v\|_N\le\|\mathbf v\|_a.
 \label{eq:theory-assembly-lemma}
\end{equation}
\end{lemma}
\emph{Proof.}
The first equality follows pointwise from
$\sum_j a_j(x)v(x)=v(x)$.
The second follows by inserting $T_jv$ in
\eqref{eq:theory-assembly-def} and summing the weights.
For the last inequality,
convexity of the squared Euclidean norm gives, at each $x$,
\[
 \left\|\sum_j a_j(x)v_j(x)\right\|_2^2
 \le\sum_j a_j(x)\|v_j(x)\|_2^2.
\]
Summation and division by $qn$ complete the proof.

In particular, the assembled prediction error is bounded by the coverage-weighted
patch error.

\subsection{Parent error and the coarse closure residual}
\label{app:eps-parent}
Direct supervision of the Parent does not make coarse evolution closed.
For a square-integrable forecast, let $m_{\mathrm c}(z)=\mathbb E[Z^+\mid Z=z]$.
The conditional-mean projection identity gives
\begin{equation}
 \epsilon_P^2:=\mathbb E\|P_\phi(Z)-Z^+\|_R^2
 =\sigma_{\mathrm{cl}}^2+\mathbb E\|P_\phi(Z)-m_{\mathrm c}(Z)\|_R^2.
 \label{eq:eps-parent}
\end{equation}
The terms separate predictive uncertainty given the coarse observation from error relative to its conditional mean.
We therefore pass the forecast as a condition whose errors the Child can correct in its output, rather than requiring the output to reproduce it.

\subsection{One-step error decomposition and Parent replacement}
\label{sec:analysis}\label{app:analysis}
This section quantifies how forecast errors enter the assembled field.
We fix the trained maps, preprocessing, assembly weights, and evaluation law, and compare the deployed predictor with the same Child supplied with the true coarse future.
Using $X_j=T_jU$, $C_j=c_j(U)$, and $C_j^*=T_jIZ^+=T_jIDV$, define
\begin{equation}
 F_*(U,V)=A\big((C_\vartheta(X_j,C_j^*))_j\big).\label{eq:reference-map}
\end{equation}
The target-dependent $F_*$ is a diagnostic, not a deployable predictor or an unconditional risk lower bound.
The deployed error separates exactly into the diagnostic residual and the change induced by the forecast:
\begin{equation}
 F(U)-V
 =\underbrace{F_*(U,V)-V}_{b}
  +\underbrace{F(U)-F_*(U,V)}_{d}.
 \label{eq:error-telescope}
\end{equation}
We summarize the two terms and their interaction by
\begin{equation}
 \rho^2=\mathbb E\|b\|_N^2,\qquad
 \eta^2=\mathbb E\|d\|_N^2,\qquad
 \chi=\mathbb E\langle b,d\rangle_N.
 \label{eq:residual-components}
\end{equation}

\begin{samepage}
\begin{theorem}[One-step decomposition and conditional stability]
\label{thm:composition}
Assume the fixed weights satisfy Equation~\eqref{eq:assembly}, $I$ is linear, and $\rho,\epsilon_P<\infty$.
For each patch, suppose a finite deterministic $\ell_j$ satisfies, for $\mu$-almost every $(U,V)$,
\begin{equation}
 \|C_\vartheta(X_j,C_j)-C_\vartheta(X_j,C_j^*)\|_2
 \le\ell_j\|C_j-C_j^*\|_2,
 \label{eq:lipschitz}
\end{equation}
The norms here are ordinary Euclidean patch norms, and the constants apply to the actual/oracle condition pair just defined.
Set $\kappa_I=\|I\|_{R\to N}$, $a_j^{\max}=\max_xa_j(x)$, and
\begin{equation}
 L_{\mathrm{cov}}^2=\max_x\sum_{j:x\in\Omega_j}a_j^{\max}\ell_j^2.
 \label{eq:coverage}
\end{equation}
Then
\begin{align}
 \mathcal R_\mu(F)&=\rho^2+\eta^2+2\chi,\label{eq:exact-decomposition}\\
 \sqrt{\mathcal R_\mu(F)}&\le\rho+\eta
 \le\rho+L_{\mathrm{cov}}\kappa_I\epsilon_P.\label{eq:composition-bound}
\end{align}
\end{theorem}
\par\end{samepage}
The proof first controls each patch response to changing its condition, then uses the nonnegative assembly weights and the triangle inequality (Appendix~\ref{app:theory-composition}).
A small reference residual, an accurate Parent, and controlled Child sensitivity are sufficient to make this upper bound small.
They are not necessary for small actual error, since the two residuals can cancel.
Fixed normalization scales are included in the physical maps and hence in their sensitivity constants.

The exact identity also matters: $\chi$ can have either sign, so $\rho^2$ and $\eta^2$ are not additive causal shares of the total error.
Paired held-out predictions with actual and true coarse-future conditions, compared in the norm of Equation~\eqref{eq:theory-norms}, yield empirical estimates of $\rho,\eta,\chi$.

\paragraph{Parent replacement.}
Separate supervision permits a new Parent to be connected to the fixed Child without changing its weights.
For two compatible Parents and the same fixed Child, assume finite composed risks and the direct new/old sensitivity condition of Corollary~\ref{cor:replacement}, with $L_{\mathrm{cov}}$ formed from its constants.
With $\delta_P^2=\mathbb E\|P_{\mathrm{new}}(DU)-P_{\mathrm{old}}(DU)\|_R^2<\infty$,
\begin{equation}
 \left|\sqrt{\mathcal R_\mu(F_{\mathrm{new}})}-\sqrt{\mathcal R_\mu(F_{\mathrm{old}})}\right|
 \le L_{\mathrm{cov}}\kappa_I\delta_P.
 \label{eq:replacement-stability}
\end{equation}
Corollary~\ref{cor:replacement} shows that a Parent close to the old one changes the error of the fixed Child only slightly.
We compare complete-field predictions on paired inputs and account for additional Parent data and training separately.

\subsection{Composition bound}
\label{app:theory-composition}

\paragraph{Assumptions and restatement.}
Use the weights and norms above, and let
\begin{equation}
 \kappa_I=\sup_{z\ne0}\frac{\|Iz\|_N}{\|z\|_R}<\infty.
 \label{eq:theory-interpolation-constant}
\end{equation}
For the raw coordinate matrix $I_{\mathrm{mat}}$, the common channel count gives
$\kappa_I=\sqrt{r/n}\,\|I_{\mathrm{mat}}\|_{2\to2}$.
No isometry property is assumed for the implementation's interpolation.

Use the ordinary Euclidean norm on each patch, as in Theorem~\ref{thm:composition}; the field normalization enters when patch errors are assembled.
Assume deterministic finite constants $\ell_j$ such that, for
$\mu$-almost every $(U,V)$ and
$(c,c')=(T_jIP(DU),T_jIDV)$,
\begin{equation}
 \|C_\vartheta(T_jU,c)-C_\vartheta(T_jU,c')\|_2
 \le\ell_j\|c-c'\|_2.
 \label{eq:theory-child-lipschitz}
\end{equation}
Only these paired conditions are needed for the theorem.
If the inequality is
established by integrating a Jacobian bound, that bound must also cover the
connecting paths.
Define
\begin{equation}
 a_j^{\max}=\max_x a_j(x),
 \qquad
 L_{\mathrm{cov}}^2
 =\max_x\sum_{j:x\in\Omega_j}a_j^{\max}\ell_j^2,
 \label{eq:theory-coverage-constant}
\end{equation}
and define the Parent error and a supplementary patch-residual quantity:
\begin{align}
 \epsilon_P^2
 &=\mathbb E\|P(DU)-DV\|_R^2,
 \label{eq:theory-parent-error}\\
 \epsilon_C^2
 &=\mathbb E\left\|
   \big(C_\vartheta(T_jU,T_jIDV)-T_jV\big)_j
   \right\|_a^2.
 \label{eq:theory-child-oracle-error}
\end{align}
Theorem~\ref{thm:composition} assumes $\epsilon_P<\infty$ and $\rho<\infty$, with $\rho$ defined below; it does not require $\epsilon_C<\infty$.
When finite, $\epsilon_C$ gives an additional upper bound on $\rho$.
It is the oracle-conditioned residual of the
\emph{fixed} Child.
It is neither an optimized Bayes risk nor an assumed lower
bound on the deployed error.
Define the assembled diagnostic and residuals by
\begin{align}
 F_*(U,V)&=A\big((C_\vartheta(T_jU,T_jIDV))_j\big),\nonumber\\
 b&=F_*(U,V)-V,\qquad d=F_P(U)-F_*(U,V),
 \label{eq:theory-residual-components}\\
 \rho^2&=\mathbb E\|b\|_N^2,\qquad
 \eta^2=\mathbb E\|d\|_N^2,\qquad
 \chi=\mathbb E\langle b,d\rangle_N.\nonumber
\end{align}
The diagnostic $F_*$ depends on the target and is not a deployable predictor.
Equation~\eqref{eq:theory-child-lipschitz} is the same paired Euclidean sensitivity condition used in Theorem~\ref{thm:composition}.
The proof of Theorem~\ref{thm:composition} also gives the following exact decomposition and stability bounds
\begin{align}
 \mathcal R_\mu(F_P)&=\rho^2+\eta^2+2\chi,
 \label{eq:theory-exact-decomposition}\\
 \rho&\le\epsilon_C,\qquad
 \eta\le L_{\mathrm{cov}}\kappa_I\epsilon_P,
 \label{eq:theory-component-bounds}\\
 |\rho-\eta|\le\sqrt{\mathcal R_\mu(F_P)}
 &\le\rho+\eta\le\rho+L_{\mathrm{cov}}\kappa_I\epsilon_P.
 \label{eq:theory-composition-bound}
\end{align}

\paragraph{Proof of Theorem~\ref{thm:composition}.}
Pointwise in $(U,V)$, write $\widehat z^+=P(DU)$, $z^+=DV$, and
\begin{equation}
 \delta_j=C_\vartheta(T_jU,T_jI\widehat z^+)-C_\vartheta(T_jU,T_jIz^+).
 \label{eq:theory-oracle-assembly}
\end{equation}
By the weighted-assembly lemma and $a_j(x)\le a_j^{\max}$,
\begin{align}
 \|d\|_N^2
 &=\|F_P(U)-F_*(U,V)\|_N^2\nonumber\\
 &\le\|(\delta_j)_j\|_a^2\nonumber\\
 &\le\frac{1}{qn}\sum_j a_j^{\max}\|\delta_j\|_2^2\nonumber\\
 &\le\frac{1}{qn}\sum_j a_j^{\max}\ell_j^2\|T_jI(\widehat z^+-z^+)\|_2^2\nonumber\\
 &=\frac{1}{qn}\sum_x
     \left(\sum_{j:x\in\Omega_j}a_j^{\max}\ell_j^2\right)
     \|I(\widehat z^+-z^+)(x)\|_2^2\nonumber\\
 &\le L_{\mathrm{cov}}^2\|I(\widehat z^+-z^+)\|_N^2
 \le L_{\mathrm{cov}}^2\kappa_I^2\|\widehat z^+-z^+\|_R^2.
 \label{eq:theory-condition-perturbation}
\end{align}
The Lipschitz step uses \eqref{eq:theory-child-lipschitz} and the linearity of $I$.
Also, $A TV=V$ implies
\begin{equation}
 \|b\|_N^2=\|F_*(U,V)-V\|_N^2
 \le\left\|\big(C_\vartheta(T_jU,T_jIDV)-T_jV\big)_j\right\|_a^2.
 \label{eq:theory-oracle-residual-bound}
\end{equation}
Taking expectations gives
$\eta\le L_{\mathrm{cov}}\kappa_I\epsilon_P$ and $\rho\le\epsilon_C$.
The sensitivity bound and $\epsilon_P<\infty$ imply $d\in L^2(\mu;\mathcal H_N)$; the assumption $\rho<\infty$ gives the same property for $b$.
Thus
Cauchy--Schwarz gives $|\chi|\le\rho\eta<\infty$.
Since
$F_P(U)-V=b+d$, expansion in this Hilbert space yields
\[
 \mathcal R_\mu(F_P)=\mathbb E\|b+d\|_N^2
 =\mathbb E\|b\|_N^2+\mathbb E\|d\|_N^2
   +2\mathbb E\langle b,d\rangle_N
 =\rho^2+\eta^2+2\chi.
\]
The triangle and reverse triangle inequalities in the same space give
\[
 |\rho-\eta|
 \le\|b+d\|_{L^2(\mu;\mathcal H_N)}
 \le\rho+\eta
 \le\rho+L_{\mathrm{cov}}\kappa_I\epsilon_P,
\]
which completes the proof.
If $\epsilon_C<\infty$, one may further replace $\rho$ by $\epsilon_C$ in the upper bound.
The explicit finiteness of $\rho$ and $\epsilon_P$ justifies the $L^2$ argument; finite moments of the inputs alone would not suffice
for arbitrary nonlinear predictors.

The cross term $\chi$ can have either sign.
Thus $\rho^2$ and $\eta^2$ are not
additive causal shares of the actual error.
Paired evaluations of the same fixed
Child can compute empirical versions of all three terms.

\paragraph{Role of coverage and sensitivity.}
If $\ell_j\le\ell$ and no voxel belongs to more than $m_{\max}$ patches, then
$L_{\mathrm{cov}}\le\ell\sqrt{m_{\max}}$.
For constant multiplicity $m$ with
$a_j(x)=1/m$ on each patch, the sharper definition gives
$L_{\mathrm{cov}}\le\ell$.
The proof deliberately uses ordinary patch norms rather than assuming that the
Child is Lipschitz in the coverage-weighted seminorm.
A patch coordinate with zero
assembly weight can still influence positive-weight outputs through the network.

\subsection{Parent replacement}
\label{app:theory-replacement}

\begin{corollary}[Fixed-Child replacement stability]
\label{cor:replacement}
Fix the Child, linear interpolation, evaluation law, and assembly weights
satisfying Equation~\eqref{eq:assembly}.
For two compatible Parents, write $F_s=F_{P_s}$ for $s\in\{\mathrm{old},\mathrm{new}\}$ and assume
$\mathcal R_\mu(F_{\mathrm{old}}),\mathcal R_\mu(F_{\mathrm{new}})<\infty$.
Write $c_j^s=T_jIP_s(DU)$ for $s\in\{\mathrm{old},\mathrm{new}\}$.
Suppose finite deterministic $\ell_j$ satisfy, for $\mu$-almost every $(U,V)$,
\[
 \|C_\vartheta(T_jU,c_j^{\mathrm{new}})
   -C_\vartheta(T_jU,c_j^{\mathrm{old}})\|_2
 \le\ell_j\|c_j^{\mathrm{new}}-c_j^{\mathrm{old}}\|_2.
\]
This is a condition on the direct new/old pair.
Define $L_{\mathrm{cov}}$
using these constants in Equation~\eqref{eq:theory-coverage-constant}, and let
\begin{equation}
 \delta_P^2
 =\mathbb E\|P_{\mathrm{new}}(DU)-P_{\mathrm{old}}(DU)\|_R^2<\infty.
 \label{eq:theory-parent-disagreement}
\end{equation}
Then
\begin{align}
 \big|\sqrt{\mathcal R_\mu(F_{\mathrm{new}})}-\sqrt{\mathcal R_\mu(F_{\mathrm{old}})}\big|
 &\le\|F_{\mathrm{new}}-F_{\mathrm{old}}\|_{L^2(\mu;\mathcal H_N)}\nonumber\\
 &\le L_{\mathrm{cov}}\kappa_I\delta_P.
 \label{eq:theory-replacement-stability}
\end{align}
\end{corollary}
\emph{Proof.}
Apply the weighted-assembly calculation in Equation~\eqref{eq:theory-condition-perturbation} to the direct pair of forecasts $P_{\mathrm{new}}(DU)$ and $P_{\mathrm{old}}(DU)$.
This calculation uses the assumed new/old sensitivity condition and gives
\[
 \|F_{\mathrm{new}}(U)-F_{\mathrm{old}}(U)\|_N^2
 \le L_{\mathrm{cov}}^2\kappa_I^2
 \|P_{\mathrm{new}}(DU)-P_{\mathrm{old}}(DU)\|_R^2.
\]
Taking expectations and square roots proves the second inequality.
Applying
the reverse triangle inequality to
$F_{\mathrm{new}}(U)-V$ and $F_{\mathrm{old}}(U)-V$ in
$L^2(\mu;\mathcal H_N)$ proves the first.
Forecast disagreement can be evaluated on inputs without future labels.

\paragraph{Exact risk change.}
Let
$e=F_{\mathrm{old}}(U)-V$ and
$\Delta=F_{\mathrm{new}}(U)-F_{\mathrm{old}}(U)$.
Expanding the square gives the
exact identity
\begin{equation}
 \mathcal R_\mu(F_{\mathrm{new}})-\mathcal R_\mu(F_{\mathrm{old}})
 =2\mathbb E\langle e,\Delta\rangle_N
  +\mathbb E\|\Delta\|_N^2.
 \label{eq:theory-replacement-identity}
\end{equation}
Thus strict improvement holds if and only if
\begin{equation}
 \mathbb E\langle e,\Delta\rangle_N
 <-\frac12\mathbb E\|\Delta\|_N^2.
 \label{eq:theory-replacement-criterion}
\end{equation}